\RequirePackage[l2tabu, orthodox]{nag}
\documentclass{colt2016} % Include author names
\pdfoutput=1

\usepackage[english]{babel}
\usepackage[utf8]{inputenc}
\usepackage[T1]{fontenc}

\usepackage{amsfonts}

\usepackage{enumerate}
\usepackage{array}
\usepackage{booktabs}

\usepackage[capitalize]{cleveref}
\usepackage{breakcites}

\usepackage{thmtools}
\usepackage{thm-restate}

\usepackage{tikz-cd}
\usetikzlibrary{backgrounds}
\tikzset{
  every plot/.style={prefix=plots/pgf-},
  shape example/.style={
    color=black!30,
    draw,
    fill=yellow!30,
    line width=.5cm,
    inner xsep=2.5cm,
    inner ysep=0.5cm}
}
\definecolor{graphicbackground}{rgb}{0.96,0.96,0.8}
\pgfkeys{/tikz/.cd,
  background color/.initial=graphicbackground,
  background color/.get=\backcol,
  background color/.store in=\backcol,
}
\tikzset{background rectangle/.style={
    fill=\backcol,
  },
  use background/.style={
    show background rectangle
  }
}

\makeatletter
\def\thmt@refnamewithcomma #1#2#3,#4,#5\@nil{%
  \@xa\def\csname\thmt@envname #1utorefname\endcsname{#3}%
  \ifcsname #2refname\endcsname
  \csname #2refname\expandafter\endcsname\expandafter{\thmt@envname}{#3}{#4}%
  \fi
}
\makeatother

\Crefname{proposition}{Proposition}{Propositions}

\declaretheorem[name=Lemma,refname={Lemma,Lemmas},Refname={Lemma,Lemmas},sibling=theorem]{lemma}

\declaretheorem[name=Corollary,refname={Corollary,Corollaries},Refname={Corollary,Corollaries},sibling=theorem]{corollary}

\declaretheorem[name=Assumption,refname={Assumption,Assumptions},Refname={Assumption,Assumptions}]{assumption}

\declaretheorem[name=Proposition,refname={Proposition,Propositions},Refname={Proposition,Propositions},sibling=theorem]{proposition}

\newcommand{\email}[1]{\tt #1}

\DeclareMathOperator*{\argmax}{argmax}

\DeclareMathOperator*{\Lip}{Lip} % Lipschitz constant

\newcommand{\ie}{\emph{i.e.\@\xspace}}
\newcommand{\eg}{\emph{e.g.\@\xspace}}
\newcommand{\st}{s.t.\@\xspace}
\newcommand{\wrt}{w.r.t.\@\xspace}

\newcommand{\cf}{cf.\@\xspace}

\newcommand{\block}[3]{\left#1 #2 \right#3} % Generic block
\newcommand{\cb}[1]{\block{\{}{#1}{\}}} % [c]urly [b]races
\newcommand{\pb}[1]{\block{(}{#1}{)}} % [p]arentheses [b]rackets
\newcommand{\sqb}[1]{\block{[}{#1}{]}} % [sq]uare [b]rackets
\newcommand{\ab}[1]{\block{|}{#1}{|}} % [ab]solute value
\newcommand{\ip}[1]{\block{\langle}{#1}{\rangle}} % [i]nner [p]roduct
\newcommand{\tup}[1]{\ip{#1}} % [tup]le
\newcommand{\norm}[1]{\| #1\|} % norm
\newcommand{\normb}[1]{\block{\|}{#1}{\|}} % norm block
\newcommand{\mnormb}[1]{\bigl\| #1 \bigr\|} % norm block

\newcommand{\EE}[1]{\mathbb{E}\sqb{#1}} % Expectation

\newcommand{\sn}[1]{\mathcal{#1}} % [s]et [n]ame
\newcommand{\fun}[3]{#1 : #2 \rightarrow #3} % [fun]ction name
\newcommand{\eqdef}{\doteq} % definition
\newcommand{\tra}{^\top} % transpose
\newcommand{\tkern}[1]{\sn{#1}} % transition kernel
\newcommand{\rl}[1]{\mathbb{R}^{#1}} % Euclidean space, real numbers
\newcommand{\ind}[1]{\mathbb{I}\cb{#1}}

\newcommand{\Qkern}{\tkern{Q}}
\newcommand{\proj}{\tkern{R}}
\newcommand{\projA}{\proj^{\sn{A}}}
\newcommand{\Pkern}{\tkern{P}}
\newcommand{\pihat}{\hat{\pi}}

\newcommand{\als}[1]{\begin{align*} #1 \end{align*}} % Align-star
\renewcommand{\epsilon}{\varepsilon}

\title[Policy Error Bounds for MBRL with Factored Linear Models]{Policy Error Bounds for Model-Based Reinforcement Learning with Factored Linear Models}
\usepackage{times}

 \coltauthor{\Name{Bernardo {\'Avila Pires}}  \Email{bpires@ualberta.ca}\\
 \addr Department of Computing Science\\
 University of Alberta
 \AND
 \Name{Csaba Szepesv\'ari} \Email{szepesva@cs.ualberta.ca}\\
 \addr Department of Computing Science\\
 University of Alberta
 }

\begin{document}

\maketitle

\begin{abstract}

In this paper we study a model-based approach to calculating approximately optimal policies in Markovian Decision Processes.
In particular, we derive novel bounds on the loss of using a policy derived from a
factored linear model, a class of models which generalize numerous previous models out of those that come with 
strong computational guarantees. 
For the first time in the literature, we derive performance bounds for model-based techniques
where the model inaccuracy is measured in weighted norms.
Moreover, our bounds show a decreased sensitivity to the discount factor and, unlike similar bounds derived 
for other approaches, they are insensitive to measure mismatch.
Similarly to previous works, our proofs are also based on contraction arguments, 
but with the main differences that
we use carefully constructed norms building on Banach lattices,
and the contraction property is only assumed for operators acting on ``compressed'' spaces, 
thus weakening previous assumptions,
while strengthening previous results.

\end{abstract}

\section{Introduction}
\label{sec:introduction}
The recent years have witnessed a renewed interest in model-based reinforcement learning (MBRL).
\citet{barreto2011reinforcement,kveton2012kernel} and \citet{precup2012online}, building on the seminal work of \citet{ormoneit2002kernel},
studied various approaches to stochastic factorizations of the transition probability kernel, while
\citet{grunewalder2012modelling} proposed to use  RKHS embeddings to approximate the transition kernel,
with further enhancements proposed recently by \citet{lever2016compressed}.
A key common feature of these otherwise distant-looking
works is that once the model is set up, it leads to a policy in a computationally efficient way (i.e., in poly-time and space
in the size of the model).
Having realized that this is not a mere coincidence, \citet{yao2014pseudo} introduced the concept of factored linear models,
which keeps the advantageous computational properties, while generalizing all previous works.
While efficient computation is a necessity, efficient learning and good performance of the policy are equally important.
In this paper we focus on the second of these criteria, namely the performance of the policy derived from the model.
The argument for omitting the learning part for the time being
is that one should better understand first what errors need to be controlled
because this will influence the choice of the learning objective and hence the algorithms
 (we also note in passing that, in the above-mentioned examples, the statistical analysis of the model learning algorithms is well understood
by now).

We are not the first to consider the performance  of the policy as a function of the model errors.
In fact, most of the previously mentioned works also give bounds on the policy error
(we define the policy error to be the performance loss due to using the derived policy instead of an optimal one).
However, all these previous works derive bounds that express model errors in a supremum norm.
While the supremum norm is a convenient choice when working with Markovian Decision Processes (which
give the theoretical foundations in these works), an observation that goes back to at least \citet{whitt1978approximations},
the supremum norm is also known to be a rather unforgiving metric:
In learning settings, when data comes from a large cardinality set, and the data may have an uneven distribution,
while the objects of interest lack appropriate smoothness, or other helpful structural properties,
we expect errors measured in the supremum norm to decrease rather slowly.
Furthermore, most learning algorithms aim to reduce some weighted norms, hence deriving bounds for the supremum norm is neither
natural, nor desirable.
Can existing bounds of the policy error from the MBRL literature be extended to other norms?
In the analogue context of approximate dynamic programming methods, \citet{munos2003error} pioneered a technique to allow the use
of weighted $L^p$-norms to bound the policy error,
while in the context of approximate linear programming, \cite{de2003linear} proposed a different technique
to allow the use of weighted supremum norms, both leading to substantial further work \citep{BuBa10:RLBook}, 
\citep[Chapter 3]{WieMa12:RL}.
While the use of weighted norms is a major advance, these bounds do not come without any caveats.
In particular,  in ALP, the bounds rely on the similarity of the so-called constraint sampling distribution to
the stationary distribution $\mu^*$ of the optimal policy, while in ADP they rely on the similarity of the data sampling distribution and the start-state distribution,
leading to hard to control error terms. Can this be avoided by model-based approaches?

\paragraph{Contributions.}
We derive bounds on the policy error of policies derived from factored linear models in MBRL.
The policy error is bounded in supremum, weighted supremum and weighted $L^p$ norms (\cref{thm:inftyNormBound,thm:weightedInftyNormBound,thm:muNormBound}).
% and we prove an additional $L^p(\mu)$ norm derived from the weighted supremum norm bound (\cref{thm:muNormBoundInfty}).
The results hold under some conditions:
the left factor of the approximate factorization of the transition kernel
must satisfy a mild boundedness condition (\cref{ass:boundedQkern}),
the right factor must be a join-homomorphism (\cref{ass:joinHomomorphism}),
the operator obtained by swapping the left and right factors must satisfy a boundedness condition (\cref{ass:boundedNorm} or \cref{ass:Lyapunov}).
This latter condition is not mild as the one on the left factor, but it i) generalizes the conditions used to derive previous policy error bounds; and ii) can be easier to enforce as it constrains the norm of a low-dimensional operator, unlike the analogue constraints in previous works.

We recover results for unfactored linear models that satisfy a contraction assumption, and we recover existing supremum norm bounds for factored linear models that meet \cref{ass:joinHomomorphism}. In addition to being able to recover previous results, we also provide a new type of analysis,
which has interesting implications. The new analysis shows that MBRL can in fact escape the sensitivities in ALP and ADP (cf. \cref{thm:muNormBound}, term $\varepsilon_1$), answering the above major question on the positive. In fact, the new bound also shows the potential for better scaling with the discount factor, which is another surprising result.
We attribute this success to the systematic use of the language of Banach lattices, which forced us to discover amongst other things a definition of \emph{mixed} norms for action-value functions which is general, yet makes the so-called value selection operators non-expansions (cf. \cref{prop:LipMaxOperatorAlt}).
For the skeptics who believe that MBRL is ``hard'' because the derived policy cannot be good before the model approximates ``reality'' uniformly everywhere, we point out that already the first ever bound derived for policy error in MBRL (due to \citet{whitt1978approximations}) shows
that the model has to be accurate only in an extremely localized way.
Our bounds also share this characteristic of previous bounds.

Our analysis builds on techniques borrowed from approximate policy iteration (API) and approximate linear programming (ALP),
and provide new insights to existing results for ALP (\cref{prop:Lyapunov}).
However, the MBRL setup we consider is nevertheless different from API and ALP, so the connections in our proofs are not a mere translation of API or ALP results to MBRL, as we will explain in \cref{sec:related}, which is also attested by the novel features of our bounds.

Other miscellaneous contributions include an example that shows why controlling the deviation between the optimal value function underlying the true and approximate models, a metric often used in some previous works to evaluate model quality,
is insufficient to derive a policy error bound (\cref{prop:errorGaps,thm:adp}).
We present a characterization of linear join-homomorphisms (\cref{prop:onlyLJH}).
We show that our supremum norm bounds are tight to arbitrary accuracy (\cref{prop:tightness}), but that quantifying policy error in supremum norm can be harsh, so it pays off to consider the policy error in $L^p(\mu)$ norm instead (\cref{prop:infNormHarsh}).

The rest of the paper is organized as follows:
We start by providing the necessary background on MDPs in \cref{sec:background}, followed by introducing
 factored linear models and the questions studied in \cref{sec:approach}.
After this, we state our assumptions in \cref{sec:ass}, present our main results in \cref{sec:results},
and close with placing our work in the context of existing work, and providing an outlook for future work in \cref{sec:related}.
While we include the proofs of our main results in the main body of the paper, proofs of technical results are relegated to the appendix.

\section{Markov Decision Processes}
\label{sec:background}
We shall describe the agent-environment interaction using the framework of \emph{Markov Decision Processes} (MDPs), with which the reader is assumed to be familiar.
The notation used here is perhaps closest to that of \citet{szepesvari2010algorithms},
but the reader may also consult the books of \citet{puterman94} and \citet{sutton1998reinforcement} on background.
Here, we describe only the main concepts so as to clarify our notation.
Well-understood technical details (such as measurability) are (mostly) omitted for brevity.
The first two paragraphs describe standard notation, while the rest of the section defines less standard notation
which is essential to understand the paper.

\paragraph{Markov Decision Processes.} An MDP is a tuple $\tup{ \sn{X}, \sn{A}, \Pkern, r }$,
 where $\sn{X}$ is the state space, $\sn{A}$ is the action space,
$\Pkern= (\sn{P}^a)_{a \in \sn{A}}$ is the \emph{transition probability kernel} and
$r = (r^a)_{a \in \sn{A}}$ is the \emph{reward function}.
For each state $x\in \sn{X}$ and action $a\in \sn{A}$, $\Pkern^a(\cdot | x)$ gives
a distribution over the states in $\sn{X}$,
interpreted as the distribution over the next states given that action $a$ is taken in state $x$.
For each action $a\in \sn{A}$ and state $x\in \sn{X}$, $r^a(x)$ gives a real number,
which is interpreted as the reward received when action $a$ is taken in state $x$.%
\footnote{The standard MDP definitions would allow stochastic rewards, which may also be correlated with the next state.
Our simplified model enhances clarity and
extending our results to the case of stochastic rewards is trivial under a suitable set of assumptions.}

An MDP describes the interaction of an agent and its environment. The interaction happens in a sequential manner
where in each step the agent chooses an action $A_t\in \sn{A}$ based on the past information it has,
sends the action to the environment,
which then moves from the current state $X_t$ to the next one
according to the transition kernel: $X_{t+1}\sim \Pkern^a(\cdot|X_t)$.
The agent then observes the next state and the reward associated with the transition.
In this paper we assume that the agent's goal is to maximize
the \emph{expected total discounted reward}, $\EE{\sum_{t=0}^\infty \gamma^t r^{A_t}(X_t)}$,
where $0\le \gamma <1$ is the so-called \emph{discount factor}.
A rule describing the way an agent acts given its past actions and observations is called a \emph{policy}.
The \emph{value} of a policy $\pi$ in a state $x$, denoted by $V^\pi(x)$,
is the expected total discounted reward when the initial state ($X_0$) is $x$ assuming the agent follows the policy.
An \emph{optimal policy} is one that achieves the maximum possible value amongst all policies in each state $x\in \sn{X}$.
The \emph{optimal value} for state $x$ is denoted by $V^*(x)$.
A \emph{deterministic Markov policy} disregards everything in the history except the last step.
Such policies can and will be identified with a map $\pi: \sn{X} \to \sn{A}$, and the space of measurable deterministic Markov policies will be denoted by $\Pi$.
We will assume that the action set is finite.
When, in addition, the reward function is bounded, which we assume from now on,
all the value functions are bounded and
one can always find a deterministic Markov policy that is optimal \citep{puterman94}.
The suboptimality or \emph{policy error} of a policy $\pi$ at a state $x$ is the difference $V^*(x)-V^\pi(x)$.
Loosely speaking, a policy is near-optimal when this difference is small for the states that one cares about.
In this work we are interested in bounding the policy error (for policies described in \cref{sec:approach}) in different norm choices:
supremum norm, a weighted supremum norm and an $L^p(\mu)$ norm.

\paragraph{Spaces of value functions.} Let $(\sn{V},\norm{\,\cdot\,}_{\sn{V}})$ be a Banach space of real-valued measurable functions over $\sn{X}$, equipped with a given norm, and $(\sn{V}^{\sn{A}}, \norm{\,\cdot\,}_{\sn{V}^{\sn{A}}})$ be a Banach space of all measurable functions mapping $\sn{A}$ to $\sn{V}$.
Elements of $\sn{V}$ are called \emph{value functions},
while elements of $\sn{V}^{\sn{A}}$ are called \emph{action-value functions}.
Oftentimes, we will choose $\norm{\,\cdot\,}_{\sn{V}}$ to be the norms mentioned before. The choice of $\norm{\,\cdot\,}_{\sn{V}^{\sn{A}}}$ will in general depend on that of $\norm{\,\cdot\,}_{\sn{V}}$, but this will be made clear in the actual context.
Of course, $\sn{V}^{\sn{A}}$ can also be identified with the set of real-valued functions with domain $\sn{X} \times \sn{A}$ (since $\sn{A}$ is finite).
To avoid too many parentheses, for $V\in \sn{V}^{\sn{A}}$, we will use $V^a$ as an alternate notation to $V(a)$.
Conveniently, $V^a\in \sn{V}$.
With a slight abuse of notation, we denote by $\Pkern^a$ the
$\sn{V} \to \sn{V}$ right linear operator defined by $(\Pkern^a V )(x) \eqdef \int V(x') \Pkern^a(dx'|x)$ (we assume that $V\in \sn{V}$ implies integrability, hence the integrals are well defined).
We also view $\Pkern^a$ as a left linear operator, acting over the space of probability measures defined over $\sn{X}$:
$\Pkern^a: \sn{M}_1(\sn{X}) \to \sn{M}_1(\sn{X})$, $(\mu \Pkern^a)(dy) = \int \mu(dx) P^a(dy|x)$, $\mu \in \sn{M}_1(\sn{X})$.
In what follows, whenever a norm is uniquely identifiable from its argument, we will drop the index of the norm denoting the underlying space.

\paragraph{Operators.}
The \emph{Bellman return operator} \wrt{} $\Pkern$, $\fun{T_{\Pkern}}{\sn{V}}{\sn{V}^{\sn{A}}}$, is defined by
$T_{\Pkern}V \eqdef r + \gamma \Pkern V$ (the indexing of $T$ with $\Pkern$ will help us to replace $\Pkern$ with some other operator)
and the so-called \emph{maximum selection} operator $\fun{M}{\sn{V}^{\sn{A}}}{\sn{V}}$ is defined by $(MV)(x) \eqdef \max_a V^a(x)$.
Then, $MT_{\Pkern}$, corresponds to the \emph{Bellman optimality operator} \citep{puterman94}.
The optimal value function satisfies $V^* = MT_{\Pkern}V^*$ \citep{puterman94},
a non-linear fixed-point equation,
which is known as the \emph{Bellman optimality equation}.
The \emph{greedy operator} $\fun{G}{\sn{V}^{\sn{A}}}{\Pi}$, which selects the maximizing actions chosen by $M$,
is defined by $GV(x) \eqdef \argmax_a V^a(x)$ ($x \in \sn{X}$, with ties broken arbitrarily).
Recall that $GV^*$ is an optimal policy \citep{puterman94}.

\paragraph{Planning in MDPs.}
In the \emph{online planning problem}
we wish to compute, at any given state $x$, an action that a near-optimal policy would take (the attribute ``online''
signifies that one is allowed some amount of calculation for each state).
By collecting all actions at all states, a planning method defines a policy $\pihat$.
Disregarding computation, planning methods are compared by how good the policy they return is, i.e., by the policy error of $\pihat$.
One approach to efficient online planning is to use an abstract model which i) contains relevant information about the MDP, ii) can be efficiently constructed, and iii) allows $\pihat(x)$ to be computed efficiently at any state $x$.
In this work we are interested in online planning with a special type of abstract models, called \emph{factored linear models}.

\section{Factored Linear Models}
\label{sec:approach}

In this section, we define factored linear models, the core of our MBRL approach.
We also show examples of MBRL approaches that use factored linear models.

In a \emph{factored linear model} we approximate the MDP's stochastic kernel $\Pkern$ as the product of two linear operators,
$\Qkern \proj$, where $\fun{\proj}{\sn{V}}{\sn{W}}$, $\Qkern = (\Qkern^a)_{a\in \sn{A}}$ and $\fun{\Qkern^a}{\sn{W}}{\sn{V}^{\sn{A}}}$ \citep{yao2014pseudo}.
Here, $\sn{W} = (\sn{W}, \norm{\,\cdot\,}_{\sn{W}})$ is a Banach space of functions with (measurable) domain $\sn{I}$.
We will refer to elements of $\sn{W}$ as \emph{compressed value functions} and elements of $\sn{W}^{\sn{A}}$ as \emph{compressed action value functions} (and, occasionally, the corresponding spaces will also be called compressed, while the spaces $\sn{V}$ and $\sn{V}^{\sn{A}}$ will be called uncompressed).
These names come from the fact that often we will want to choose $\sn{I}$ to be ``small".
In fact, for computational reasons one should choose $\sn{I}$ to be finite, in which case $\sn{W}$ will be a finite-dimensional Euclidean space.
We also allow infinite $\sn{I}$, so that we can then use $\sn{I} = \sn{X}$ and compare the tightness of our results to existing results that consider unfactored linear models.

In this work, for simplicity, we assume that the reward function $r$ remains the same in the factored linear model (the extension of our results
to the case when the reward function is also approximated is routine).
Formally, in this work we will call a tuple of the form $\tup{ \sn{X}, \sn{A}, \Qkern, \proj, r }$ a factored linear model,
where $\Qkern$ and $\proj$ are as above.
Finally, note that we do not require that $\Qkern \proj$ is a stochastic operator.  Hence, a factored linear model defines a \emph{pseudo-MDP}
\citep{yao2014pseudo}.

We must define some additional operators in order to describe how we use factored linear models to derive policies.
The extension of $\proj$ to multiple actions, $\fun{\projA}{\sn{V}^{\sn{A}}}{\sn{W}^{\sn{A}}}$, is defined by $(\projA V)^a \eqdef \proj V$ ($a \in \sn{A}$), where $\sn{W}^{\sn{A}}$ is a Banach space of $\sn{A}\to \sn{W}$ functions analogously to $\sn{V}^{\sn{A}}$.
The Bellman return operator for $\Qkern$, written as $\fun{T_{\Qkern}}{\sn{W}}{\sn{V}^{\sn{A}}}$, is defined by $T_{\Qkern}w \eqdef r + \gamma \Qkern w$ ($w \in \sn{W}$).
We also define the shorthands $T_{\projA \Qkern} \eqdef \projA T_{\Qkern} = \projA r + \gamma \projA \Qkern$ (the equality holds by linearity of $\projA$) and $T_{\Qkern\proj} \eqdef T_{\Qkern}\proj = r + \gamma \Qkern\proj $ (by linearity of $\proj$).
Finally,  $\fun{M'}{\sn{W}^{\sn{A}}}{\sn{W}}$, the counterpart of the maximum selection operator $M$, is defined by $(M' w)(i) =\max_{a\in \sn{A}} w^a(i)$ ($i \in \sn{I}$).
The relationship between these operators is shown on \cref{fig:comDiagram}, and we collected the operators defined here in \cref{sec:listop} of the appendix into a table for easy of reference.
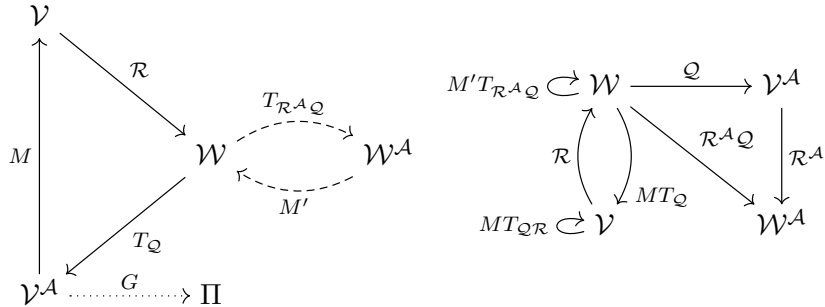
\begin{figure}[thb]
\centering
\begin{tikzcd}[column sep=large,row sep=large,background color=yellow!30]
\sn{V}
    \arrow{rd}{\proj}
    \\
& \sn{W}
    \arrow[bend left, dashed]{r}{T_{\projA\Qkern}}
    \arrow{ld}{T_{\Qkern}}
    &
\sn{W}^{\sn{A}}
    \arrow[bend left, dashed]{l}{M'}
    \\
\sn{V}^{\sn{A}}
    \arrow{uu}{M}
    \arrow[dotted]{r}{G}
& \Pi
\end{tikzcd}
\begin{tikzcd}[column sep=large,row sep=large,background color=yellow!30]
\sn{W}
    \arrow[bend left]{d}[near end]{MT_{\Qkern}}
    \arrow{dr}{\projA\Qkern}
    \arrow[loop left]{}{M'T_{\projA\Qkern}}
    \arrow{r}{\Qkern}                             & \sn{V}^{\sn{A}}
    \arrow{d}{\projA}
    \\
\sn{V}
    \arrow[bend left]{u}{\proj}
    \arrow[loop left]{}{MT_{\Qkern\proj}}             & \sn{W}^{\sn{A}}
\end{tikzcd}
\caption{Commutative diagrams showing the operators and the spaces that they act on. \label{fig:comDiagram}
}
\end{figure}

The \emph{factored linear model approach to reinforcement learning} is as follows:
Given the factored linear model $\tup{ \sn{X}, \sn{A}, \Qkern, \proj, r }$,
we take the policy
\begin{align}\label{eq:pihatdef}
\pihat \eqdef GT_{\Qkern}u^*,
\end{align}
where
\begin{align}\label{eq:ustardef}
u^* = M' T_{\projA \Qkern}u^*\,.
\end{align}
that is, the policy $\pihat$ does a \emph{Bellman lookahead} with $T_{\Qkern}$ from $u^*\in \sn{W}$, a function that satisfies a fixed-point equation.
Note that even when $\sn{X}$ is very large, or infinite, $\sn{W}$ can be finite dimensional,
in which case a good approximation to $u^*$ can
often be found in a computationally efficient manner, for example by iterating $u_{k+1} = M' T_{\projA \Qkern}u_k$, which can be seen as a form of value iteration \citep{yao2014pseudo}.
The dashed lines on the left subfigure on \cref{fig:comDiagram} show that this computation can be done
over the compressed spaces $\sn{W}$ and $\sn{W}^{\sn{A}}$. The diagram also shows that once $u^*$ is found, $T_{\Qkern}$ extends this function to $\sn{V}^{\sn{A}}$, from where using the greedy operator $G$ one obtains a policy. Note that in the applications the policy itself does not need to be explicitly represented, but the actions that the policy takes in a particular state $x\in \sn{X}$ can be computed ``on demand'' given $u^*$ and the Bellman return operator $T_{\Qkern}$.
(The right-hand side figure shows some more useful relationships between the operators involved.)
We will say that this approach is \emph{viable} when $u^*$ is well-defined.

Factored linear models allow one to analyze modeling errors in seemingly distant model-based planning methods in a \emph{unified manner}. This will be illustrated soon by describing how models proposed in numerous previous works can be written in a factored form (this was also shortly mentioned by \citet{yao2014pseudo}).
Before describing these previous models, we need some more definitions, to be able to describe the differences and similarities between them.
In particular, the models will differ in terms of whether $\proj$ is stochastic, or more specifically $\proj$ is also a point-evaluator.
Recall that the operator $\proj$ is stochastic
if $\inf_{V \geq 0}\inf_x (\proj V)(x) \geq 0$ and $\proj \mathbf{1}_{\sn{V}} = \mathbf{1}_{\sn{W}}$ where $\mathbf{1}_{\sn{V}}(x) = 1$ for all $x \in \sn{X}$ and $(\mathbf{1}_{\sn{W}})_i = 1$ for all $i \in \sn{I}$.
Here, we started to use the convention of using $w_i$ instead of $w(i)$ to reduce clutter.
Also, we say that $\proj$ is a
 \emph{point-evaluator} if $\sn{I}$ indexes elements of $\sn{X}$ and $(\proj V)_i = V(x_i)$ for all $i \in \sn{I}, V \in \sn{V}$.
Note that point evaluators are stochastic.
Choosing $\sn{I} = \sn{X}$ allows us to choose $\proj$ to be the identity,
which becomes a point evaluator when choosing $x_i = i$, $i\in \sn{I}$.

When $\proj$ is a point selector, a short direct calculation shows that $\proj M = M'\projA$, which means that on \cref{fig:comDiagram} the solid cycle and the dashed cycle starting from $\sn{W}$ are equivalent and we can interweave solid and dashed lines. For example, starting from $\sn{V}$: $MT_{\Qkern}M' T_{\projA \Qkern} \proj = (MT_{\Qkern\proj})^2$.
The equivalence $M' T_{\projA \Qkern} = \proj M T_{\Qkern}$ gives that $U^* \eqdef MT_{\Qkern}u^*$ is a fixed point of $MT_{\Qkern\proj}$, and that the identity $u^* = \proj U^*$ also holds (\cf{} \cref{thm:uStar}).
It also follows that if $M' T_{\projA\Qkern}$ is a contraction (though $MT_{\Qkern\proj}$ may not be), the factored linear model approach is viable.
To the best of our knowledge, this observation has not been made elsewhere:
In all previous works, viability was achieved by assuming that $\Qkern$ and $\proj$ are both stochastic, or that $\proj$ is a point evaluator and $\Qkern\proj$ is a non-expansion in supremum norm. (In both cases $M' T_{\projA\Qkern}$ is a contraction, so $u^*$ is well-defined and the factored linear model approach is viable.)

With this, we are ready to present different instances of the factored linear model approach:

\begin{example}[Kernel-based reinforcement learning]
\label{ex:KBRL}
In kernel-based reinforcement learning (KB\-RL), introduced by \citet{ormoneit2002kernel}, $\sn{I}$ is indexing elements of $\sn{X}$, and $\Qkern$ is a stochastic operator constructed from kernel functions at elements of $S \eqdef \cb{x_i : i \in \sn{I}}$.
Moreover,
\begin{enumerate}[(a)]
\item $S$ is an i.i.d.~sample from $\sn{X} \eqdef \rl{d}$ and $\proj$ is a point evaluator \citep{ormoneit2002kernel}; or
\item $S$ is a set of \emph{reference states} and $\proj$ is stochastic \citep{barreto2011reinforcement,kveton2012kernel,precup2012online}.
\end{enumerate}
KBRL is viable because $\Qkern$ and $\proj$ are stochastic, so $\projA\Qkern$ is also stochastic.
\end{example}

\begin{example}[Pseudo-MDPs]
\label{ex:pseudo}
Pseudo-MDPs \citep{yao2014pseudo} are factored linear models with a point evaluator $\proj$.
In pseudo-MDPs, $\Qkern$ is no longer stochastic, but $\Qkern\proj$ is assumed to be a non-expansion in supremum norm \citep{grunewalder2012modelling,yao2014pseudo,lever2016compressed}.
It can be shown that under this assumption both $MT_{\Qkern\proj}$ and $M' T_{\projA\Qkern}$ are contractions.
In the approach of these authors, one should take $\tilde{\pi} \eqdef GT_{\Qkern\proj}U^*$, where $U^*$ is the fixed point of $MT_{\Qkern\proj}$.
Our formulation still applies, though, because we can show that $u^* = \proj U^*$ is the fixed point of $M' T_{\projA\Qkern}$ (\cf{} \cref{thm:uStar}), so that $\tilde{\pi} = GT_{\Qkern\proj}U^* = GT_{\Qkern}u^* = \pihat$.
Here, $\Qkern$ is essentially learned using a penalized least-squared approach.
\end{example}

\begin{example}[State aggregation]
State aggregation \citep{whitt1978approximations,bertsekas2011approximate} in MBRL generalizes KBRL.
Here, too, $\sn{I}$ is an index set over $\sn{X}$, and $\cb{x_i : i \in \sn{I}}$ is the set of reference states.
In hard aggregation, $\proj$ is a point evaluator, while in soft aggregation \citep{SiJaJo95:aggregation} it is stochastic.
\end{example}

\begin{example}[MDP homomorphisms]
MDP homomorphisms \citep{ravindran2004algebraic,sorg2009transfer} can be used for transfer learning in reinforcement learning.
Here, $\sn{I}$ is not identified with an index set over $\sn{X}$.
If $\proj$ is a point-evaluator, we recover MDP homomorphisms \emph{per se} \citep{ravindran2004algebraic}, and the more general case of $\proj$ stochastic yields soft MRP homomorphisms \citep{sorg2009transfer}.
\end{example}

\begin{example}[Unfactored linear models]
It is possible to recover unfactored linear models as a special case of factored linear models by taking $\sn{W} = \sn{V}$, and $\proj$ to be the identity mapping.
For the approach to be viable, it is sufficient for $\Qkern$ to be stochastic, which is often assumed with unfactored linear models.
\end{example}

\section{Assumptions}
\label{sec:ass}
The purpose of this section is to state and discuss the assumptions that will be used in our subsequent results.

Our first assumption states that the operators $M: \sn{V}^{\sn{A}} \to \sn{V}$, $M': \sn{W}^{\sn{A}}\to \sn{W}$, and the related policy based value selector operators $M^\pi: \sn{V}^{\sn{A}} \to \sn{V}$ and $M'^{\pi}: \sn{W}^{\sn{A}}\to \sn{W}$ to be defined soon are non-expansions.
Operator $M^\pi$ is defined by $(M^{\pi}V)(x) \eqdef V^{\pi(x)}(x)$ ($x \in \sn{X}$, $\pi \in \Pi$), while
 $(M^{\prime\pi}w)_i \eqdef w_i^{\pi(i)}$ ($i \in \sn{I}$, $\fun{\pi}{\sn{I}}{\sn{A}}$).
Now, recall that an operator $J: \sn{E} \to \sn{F}$ mapping between Banach spaces $\sn{E} = (\sn{E},\norm{\cdot}_{\sn{E}})$,
$\sn{F} = (\sn{F},\norm{\cdot}_{\sn{F}})$ is called a non-expansion when its Lipschitz constant does not exceed one.
The Lipschitz constant of $J$ is defined by
\[
	\Lip(J) \eqdef \sup_{e,e'\in \sn{E}: e\ne e'} \frac{\norm{Je - Je'}}{\norm{e - e'}},
\]
where we follow the convention that the identity of the norm is derived from what space the argument belongs to.
Note the dependence of $\Lip$ on the norms of $\sn{E}$ and $\sn{F}$, which we suppressed. The definition implies that for any $e,e'$, $\norm{Je - Je'} \le \Lip(J)\norm{e-e'}$.
Useful properties of $\Lip$ include that it is submultiplicative ($\Lip(JJ')\le \Lip(J)\Lip(J')$),
it is invariant to constant shifts of operators ($\Lip(J+e) = \Lip(J)$, where $J+e$ is defined by $(J+e)e' = e+ Je'$) and when $J$ is a
linear operator, $\Lip(J) = \norm{J}$, the induced operator norm of $J$, which is defined by
\begin{align*}
	\norm{J} \eqdef \sup_{e\in \sn{E}, e\ne 0} \frac{\norm{Je}}{\norm{e}}.
\end{align*}
Again, the induced norm depends on the norms that the operator acts between, but we suppress this dependence.

Let us now formally state the aforementioned assumption:
\begin{assumption}[Non-expanding selectors]
\label{ass:LipMaxOperator}
We have $\Lip(M) \le 1$, $\Lip(M') \le 1$ and for any $\pi_1\in \Pi$, $\pi_2:\sn{I}\to\sn{A}$, $\Lip(M^{\pi_1}) \leq 1$ and
 $\Lip(M^{\prime\pi_2}) \leq 1$.
\end{assumption}
Note that this assumption constrains what norms can be selected for the spaces $\sn{V}^{\sn{A}}$, $\sn{V}$, $\sn{W}^{\sn{A}}$ and $\sn{W}$.
\cref{ass:LipMaxOperator} will be helpful to establish that various operators involving $M$ are Lipschitz with a factor strictly below one, i.e., that they are \emph{contractions}.
For example, to establish that $M T_{\Pkern}$ is a contraction, one can use $\Lip(MT_{\Pkern}) \le \Lip(M) \Lip(T_{\Pkern}) \le \gamma \Lip(\Pkern) = \gamma \norm{\Pkern}$, reducing the question to showing $ \gamma \norm{\Pkern}<1$. Similar arguments work the other operators that will involve $M'$, $M^\pi$, or $M'^{\pi}$.

As it was alluded to earlier, we will use a number of different norms.
However, in all cases we choose the norm for $\sn{V}^{\sn{A}}$ ($\sn{W}^{\sn{A}}$) based on the norm of $\sn{V}$
(respectively, the norm of $\sn{W}$) to be a mixed max-norm:
In particular, for $\sn{U}$ being either $\sn{V}$ or $\sn{W}$, the norm of $\sn{U}^{\sn{A}}$ will be defined as
$\norm{U}_{\sn{U}^{\sn{A}}}
	=\norm{ M_{|\cdot|} U }_{\sn{U}}$
where $M_{|\cdot|}: \sn{U}^{\sn{A}} \to \sn{U}$ is defined by
$(M_{|\cdot|} U)(\cdot) = \max_{a} |U^a(\cdot)|$. We call the resulting norm the \emph{mixed max-norm} w.r.t. the norm of $\sn{U}$.

The next proposition shows that this choice of the mixed norm makes \cref{ass:LipMaxOperator} hold whenever the underlying spaces are
so-called Banach lattices  \citep{PMB91:Banachlattices}.
Recall that a lattice is a non-empty set $\sn{U}$ with a partial ordering $\le$ such that every pair $f,g\in \sn{U}$ has a supremum (or least upper bound), denoted by $f \vee g$, and an infimum (greatest lower bound), denoted by $f \wedge  g$.
Spaces of real-valued functions are lattices with the componentwise ordering, our default choice in what follows, when it comes to $\sn{V}$ and $\sn{W}$.
Operator $\vee$ is also called a \emph{join}, a terminology we will adopt.
A vector lattice $\sn{U}$ is a lattice that is also a vector space.
In a vector lattice, for $f\in \sn{U}$, $f_+ =  f \vee 0$, $f_- = (-f) \vee 0$ and $|f| = f_+ + f_-$ (these generalize the usual definitions of positive part, negative part and absolute value).
A Banach lattice $\sn{U}$ is a normed vector lattice where $\sn{U}$ is also a Banach space and the norm satisfies that for any $f,g\in \sn{V}$, $|f|\le |g|$ $\implies$ $\norm{f}\le \norm{g}$.
With this we are ready to restate and prove the said statement:
\begin{proposition}
\label{prop:LipMaxOperatorAlt}
Assume that $\sn{V}$ and $\sn{W}$ are Banach lattices.
Then \cref{ass:LipMaxOperator} is satisfied.
\end{proposition}
\begin{proof}
To see why this holds, take for example $M$.
Then for any $U,V\in \sn{V}^{\sn{A}}$, $M U - M V \le M_{|\cdot|} (U-V)$ ($\le$ denotes the componentwise ordering)
and by swapping the order of $U,V$, we also get $|M U - M V| \le M_{|\cdot|} (U-V)$.
Now, since for any $f,g\in \sn{V}$, $|f|\le |g|$ implies $\norm{f} \le \norm{g}$,
we get $\norm{MU - M V} \le \norm{ M_{|\cdot|} (U-V) } = \norm{ U- V }_{\sn{V}^{\sn{A}}}$.
For $M^{\pi}$, since it is a linear operator, $\Lip(M^{\pi}) = \normb{M^{\pi}}$, and for any $V \in \sn{V}^{\sn{A}}$, $\ab{M^{\pi}V^{\sn{A}}} \leq M\ab{V^{\sn{A}}}$, so $\Lip(M^{\pi}) \leq \Lip(M) \leq 1$.
The statement is proven for the other operators analogously.
\end{proof}

Let us now define the norms we will use in this paper.
The weighted supremum norm of a function $f: \sn{Z} \to \rl{}$ with respect to weight $w:\sn{Z} \to \rl{}_+$ is defined as $\norm{f}_{\infty,w} = \sup_{z\in \sn{Z}} |f(z)|/w(z)$.
When $w = \mathbf{1}$ (i.e., $w(z)= 1$ for all $z\in \sn{Z}$), we drop $w$ from the index and use $\norm{f}_{\infty}$.
For $p\ge 1$, the $L^p(\mu)$-norm of $f$ is defined as $\norm{f}^p_{\mu,p} \eqdef \int_{\sn{Z}} \ab{f(z)}^p d\mu(z)$. 
By slightly abusing notation, the mixed norm of space $\sn{U}^{\sn{A}}$ derived from $\norm{\cdot}_{\infty,w}$, or $\norm{\cdot}_{p,\mu}$ will be denoted identically (i.e., for $V\in \sn{V}^{\sn{A}}$, $\norm{V}_{\infty,w}$ is a mixed norm defined using $M_{|\cdot|}$).
Since these norms make their underlying spaces a Banach lattice, we immediately get the following corollary to \cref{prop:LipMaxOperatorAlt}:
\begin{restatable}{corollary}{propLipMaxOperator}
\label{prop:LipMaxOperator}
Assume that the norms over $\sn{V}$ and $\sn{W}$ are supremum norms, weighted supremum norms, or $L^p(\mu)$ and $L^p(\rho)$ norms, and equip the spaces $\sn{V}^{\sn{A}}$ and $\sn{W}^{\sn{A}}$ with the respective mixed norms.
Then \cref{ass:LipMaxOperator} is satisfied.
\end{restatable}

Note that $(\sn{V},\vee)$ is a semi-lattice (a lattice with only a join).
For the sake of simplicity, we make the following assumption, which will be assumed to hold until \cref{thm:muNormBoundStochastic}.
\begin{assumption}[$\proj$ is a join-homomorphism]
\label{ass:joinHomomorphism}
The operator $\proj$ is a join-homomorphism of the semi-lattice $(\sn{V}, \vee)$ into the semi-lattice $(\sn{W}, \vee)$,
i.e., $\proj (U \vee V) = (\proj U) \vee (\proj V)$ for any $U,V\in \sn{V}$.
\end{assumption}
This assumption ensures that $\proj M = M' \projA$, an identity which  can be seen to hold simply by using the definitions and the above assumption, and which will be frequently used in our proofs.

The point evaluator defined in \cref{sec:approach} is a linear join-homomorphism, and,
since the identity operator is a point evaluator, it is also a linear join-homomorphism.
However, stochastic operators (also often used in place of $\proj$) may not be join-homomorphisms.
As it turns out, the class of linear join-homomorphisms is not very diverse.
\Cref{prop:onlyLJH} supports this claim for finite-dimensional $\sn{V}$ and $\sn{W}$, and the extension to infinite-dimensional spaces can be obtained by projection on finite-dimensional spaces.
For a positive integer $m$, we let $[m] = \{1,\dots,m\}$.
\begin{restatable}{proposition}{propOnlyLJH}
\label{prop:onlyLJH}
Assume that $\sn{V} = \rl{m}$ and $\sn{W} = \rl{n}$, and let $\proj$ be any linear join-homomorphism.
Then there exists $a \in \rl{n}_+$ and $J \in [m]^n$ \st{} $(\proj v)_i = a_i \, v_{J_i}$ for all $v \in \sn{V}$ and $i \in [n]$. 
\end{restatable}

Our subsequent assumptions will ensure that certain operators are contractions in appropriate norms.
We start with the simplest of these assumptions:
\begin{assumption}
\label{ass:boundedNorm}
The following hold for $\Qkern$ and $\projA$: $\norm{\projA \Qkern} \leq 1$.
\end{assumption}
Note that $\projA \Qkern$ is a $(\sn{W}, \norm{\cdot}_{\sn{W}}) \to (\sn{W}^{\sn{A}}, \norm{\cdot}_{\sn{W}^{\sn{A}}})$ operator
and the norm used in \cref{ass:boundedNorm} is the respective operator norm.
As mentioned earlier, whenever \cref{ass:LipMaxOperator} holds (which is the case for the norms under which we bound the policy error, \cf{} \cref{prop:LipMaxOperator}), we have that $\Lip(M' T_{\projA\Qkern}) \leq \gamma \normb{\projA \Qkern}$, and then \cref{ass:boundedNorm} implies that $M' T_{\projA\Qkern}$ is a $\gamma$-contraction (again, for the respective operator norm).
That $\projA \Qkern$ is a map between the compressed spaces $\sn{W}$ and $\sn{W}^{\sn{A}}$ is significant: When $\sn{W}$ is a finite dimensional space, \cref{ass:boundedNorm} can be enforced during a learning procedure as done, e.g., by \cite{yao2014pseudo}.
In fact, \cite{yao2014pseudo} argue by means of some examples that enforcing this constraint as opposed to enforcing $\norm{ \Qkern\proj }\le 1$
(which may be difficult to enforce as it constrains the norm of an operator between potentially infinite dimensional spaces)
can lead to better results in some learning settings.

When the norms are specifically chosen to be weighted supremum norms,
the previous assumption can be replaced by a weaker one, to be stated next.
To state this assumption, we need to introduce the concept of Lyapunov functions, building on a more specialized definition due to \citet{de2003linear}.
As \citet{de2003linear} showed by means of an example, using weighted supremum norms can greatly reduce the error bounds.
Intuitively, one achieves this by assigning large weights to unimportant states, i.e., to states that are infrequently visited by any policy.
Indeed, one should not expect much data, or a good behavior at such states, but since they are not visited often, the errors made at such states
can be safely discounted.

Given $\sn{Z}= (\sn{Z}, \normb{\,\cdot\,}_{\infty, w})$, with $\fun{w}{\sn{Z}}{\rl{}_+}$, and an operator $\fun{J}{\sn{Z}}{\sn{Z}}$, first let us define
\[
	\beta_{w, J} = \gamma \sup_{f : \ab{f} = w}\normb{J f}_{\infty,w}\,.
\]
Then, we say that the function $w$ is \emph{$\gamma$-Lyapunov} with respect to operator $J$ if $\beta_{w, J} < 1$.
We also extend the definition for operators of the form $K: \sn{Z} \to \sn{Z}^{\sn{A}}$, i.e., when $K = (K^a)_{a\in \sn{A}}$.
In this case, we say that $w$ is $\gamma$-Lyapunov \wrt{} $K$ if it is $\gamma$-Lyapunov \wrt{} each operator $K^a$ for any $a \in \sn{A}$.
If $J$ satisfies $J f \leq J \ab{f}$ for all $f \in \sn{Z}$ (\eg{}, if $J$ is a stochastic operator), then the definition of $\beta_{w, J}$ simplifies to $\gamma\normb{J w}_{\infty,w}$, coinciding with the definition of \citet{de2003linear}.

Lyapunov functions enable us to ensure that $MT_{\Pkern}$, $M^{\pi}T_{\Pkern}$ ($\pi \in \Pi$) and $M' T_{\projA \Qkern}$ are contractions in the corresponding weighted supremum norms.
For this, notice that the following hold:
\begin{restatable}{proposition}{propLyapunov}
\label{prop:Lyapunov}
Given ($\sn{U}$, $\normb{\,\cdot\,}_{\infty, \nu}$) with $\fun{\nu}{\sn{U}}{\rl{}_+}$, and $\fun{J}{\sn{U}}{\sn{U}^{\sn{A}}}$, if each $J^a$ is a linear operator, then $\gamma \Lip(J) = \beta_{\nu, J}$.
\end{restatable}
Now, if $\nu$ is $\gamma$-Lyapunov \wrt{} the probability kernel $\Pkern$, then we immediately get from \cref{prop:LipMaxOperator,prop:Lyapunov} that $MT_{\Pkern}$ and $M^{\pi}T_{\Pkern}$ (for any $\pi \in \Pi$) are $\beta_{\nu, \Pkern}$-contractions in $\nu$-weighted supremum norm.
Similarly, if $\eta$ is $\gamma$-Lyapunov \wrt{} $\projA \Qkern$, then $M' T_{\projA \Qkern}$ is a $\beta_{\eta, \projA\Qkern}$-contraction in $\eta$-weighted supremum norm.

With this, we can state the assumption that we will use to relax \cref{ass:boundedNorm} when the norms used the respective function spaces are weighted supremum norms.
In what follows we fix two functions, $\nu: \sn{V}\to \rl{}_+$ and $\eta: \sn{W} \to \rl{}_+$, which will act as weighting functions.
\begin{assumption}[Lyapunov weights]
\label{ass:Lyapunov}
The following hold for $\Qkern$, $\projA$, $\nu$, and $\eta$:
\begin{enumerate}[(i)]
\item $\nu$ is $\gamma$-Lyapunov \wrt{} $\Pkern$;\label{ass:Lyapunov:nu}
\item $\eta$ is $\gamma$-Lyapunov \wrt{} $\projA \Qkern$. \label{ass:Lyapunov:eta}
\end{enumerate}
\end{assumption}
Note that choosing the weight function $\nu$ to be the constant one function, \cref{ass:Lyapunov}\eqref{ass:Lyapunov:nu} is automatically satisfied,
while choosing $\eta$ to be the constant one function,  \cref{ass:Lyapunov}\eqref{ass:Lyapunov:eta} is equivalent to \cref{ass:boundedNorm} when the norm used there is the supremum norm.

Some (but not all) of our bounds will have a dependency on $\Lip(T_{\Qkern}) = \gamma \normb{\Qkern}$.
Therefore, we will also make \cref{ass:boundedQkern} to avoid vacuous bounds. 
\begin{assumption}
\label{ass:boundedQkern}
We have that $B \eqdef \norm{\Qkern} < \infty$.
\end{assumption}
Note that this assumption is mild: Learning procedures would more often than not guarantee finiteness of the objects they return.
In fact, by appropriate normalization, even $\norm{Q}\le 1$ can be arranged (if necessary) as done, for example, by \citet{grunewalder2012modelling}.

\section{Results}
\label{sec:results}

In this section we present our main results.
We start with a viability result (explaining why our minimal assumptions are sufficient for the existence of the policy whose performance we are interested in), followed by a short review of previous bounds on the policy error.
These previous bounds provide the context for our new results, which we present afterwards.
After each result we discuss their relative merits and present their proofs.
We reiterate that for all the results in this section \cref{ass:joinHomomorphism} is assumed to hold, \ie{}, $\proj$ is assumed to be a join-homomorphism.

\subsection{A viability result}

\Cref{thm:uStar} formalizes that $u^*$ is well-defined (the MBRL approach with factored linear models is viable) under \cref{ass:boundedNorm} or \cref{ass:Lyapunov}~\eqref{ass:Lyapunov:eta}, provided that the norm over $\sn{W}^{\sn{A}}$ is a mixed max-norm \wrt{} the norm over $\sn{W}$.
\Cref{thm:uStar} shows that $M' T_{\projA \Qkern}$ is a contraction (in $\normb{\cdot}_{\sn{W}}$) and we can compute $u^*$ by value iteration.
Therefore, as remarked in \cref{sec:background}, if the compressed space $\sn{W}$ is finite dimensional, we are able to evaluate $M' T_{\projA\Qkern}$ and thus also approximate $u^*$ efficiently (up to the desired accuracy).
Evaluating $\pihat(x)$ can be done by computing $(T_{\Qkern}u^*)(x)$ for each $x$ as needed.
\Cref{thm:uStar} also shows that $MT_{\Qkern\proj}$ has a unique fixed point $U^* = MT_{\Qkern}u^*$, and it is not hard to see that $U^*$ is a fixed point of $M^{\pihat}T_{\Qkern\proj}$ as well.
The fixed points $U^*$ and $u^*$, as well the contraction $M' T_{\Qkern \proj}$, will play pivotal roles in our bounds.
\begin{restatable}{theorem}{thmuStar}
\label{thm:uStar}
Assume that the norm over $\sn{W}^{\sn{A}}$ is the mixed max-norm \wrt{} the norm over $\sn{W}$, and let \cref{ass:boundedNorm} or \cref{ass:Lyapunov}~\eqref{ass:Lyapunov:eta} hold.
Assume also that $\proj$ satisfies \cref{ass:joinHomomorphism}.
Then $M' T_{\projA\Qkern}$ is a contraction \wrt{} the norm underlying $\sn{W}$, $M' T_{\projA\Qkern}$ has a unique fixed point $u^*$, and the iteration $u_{k+1} = M' T_{\projA\Qkern}u_k$ converges geometrically to $u^*$, for any $u_0 \in \sn{W}$.
Moreover, $U^* \eqdef MT_{\Qkern}u^*$ is the unique fixed point of $MT_{\Qkern\proj}$, and the identity $u^* = \proj U^*$ holds.
\end{restatable}

Before this work, it was not known that $M' T_{\projA \Qkern}$ being a contraction is sufficient for $MT_{\Qkern\proj}$ to have a unique fixed point.
As pointed out in \cref{sec:background}, to the best of our knowledge, all previous works either assumed or imposed a contraction property on $MT_{\Qkern\proj}$.
In fact, with the exception of \citet{yao2014pseudo}, all previous works required $\Qkern\proj$ to be stochastic.

In the proof of \Cref{thm:uStar}, which is presented ahead, we will use the following more general result:
\begin{restatable}{lemma}{lemBContractionFixedPoint}
\label{lem:bContractionFixedPoint}
Let $(\sn{V}, \norm{\,\cdot\,}_{\sn{V}})$ and $(\sn{W}, \norm{\,\cdot\,}_{\sn{W}})$ be two Banach spaces.
Let $\fun{T}{\sn{W}}{\sn{V}}$ and $\fun{H}{\sn{V}}{\sn{W}}$ be two operators such that $\Lip((HT)^m) < 1$ for some $m > 0$.
Then $HT$ has a unique fixed point $W^*$, and $V^* \eqdef TW^*$ is the unique fixed point of $TH$.
\end{restatable}

The proof of \cref{lem:bContractionFixedPoint} can be found in \cref{sec:generalResults}.
The argument we use is intuitive when $m = 1$: If $HT$ is a contraction, it has a fixed point $W^*$, so defining $V^* \eqdef TW^*$ gives $V^* = TW^* = THTW^* = THV^*$, so $V^*$ is a fixed point of $TH$ (and we also have the identity $W^* = HV^*$).
The operator $TH$ need not be a contraction for $V^*$ to be its fixed point; indeed, we can even have $\Lip(TH) = \infty$ and still have $V^* = THV^*$ (\cf{} \cref{prop:TQRContraction}).
The argument for $m > 1$ and for ensuring uniqueness relies largely on Banach's fixed point theorem.

\begin{proof} (of \cref{thm:uStar}).
To prove \cref{thm:uStar}, we can apply \cref{lem:bContractionFixedPoint} with $m = 1$, $T = MT_{\Qkern}$ and $H = \proj$, but we have to ensure that $\Lip(M' T_{\projA\Qkern}) < 1$.
We can use submultiplicativity of $\Lip$ and affinity of $T_{\projA\Qkern}$ to get that $\Lip(M' T_{\projA \Qkern}) \leq \gamma \Lip(M')\Lip(\projA \Qkern)$.
By the choice of norm over $\sn{W}^{A}$, $\Lip(M') \leq 1$, and by assumption $\gamma \Lip(\projA \Qkern) < 1$, so, indeed, $\Lip(M' T_{\projA \Qkern}) < 1$.

So far we have established that $M' T_{\projA \Qkern}$ is a contraction, and \cref{lem:bContractionFixedPoint} gives us that $u^*$ is the fixed point of $M' T_{\projA \Qkern}$, that $U^*$ is the fixed point of $M T_{\Qkern\proj}$, and that the two fixed points satisfy $u^* = \proj U^*$.
Because $M' T_{\projA \Qkern}$ is a contraction, the iteration $u_{k+1} = M' T_{\projA\Qkern}u_k$ converges geometrically to $u^*$, for any $u_0 \in \sn{W}$, by Banach's fixed-point theorem.
\end{proof}

\subsection{Previous results on the policy error}
\label{sec:baseline}

The typical MBRL performance bound is a supremum-norm bound on the policy error of $\tilde{\pi} \eqdef GT_{\widetilde{\Pkern}}\widetilde{V}$, where $\widetilde{\Pkern}$ is stochastic and $\widetilde{V}$ is the fixed point of $MT_{\widetilde{\Pkern}}$.
\begin{theorem}[Baseline bound on MBRL policy error]
\label{thm:baseline}
Consider some transition probability kernel $\widetilde{\Pkern}$ for the state and action spaces $\sn{X}$ and $\sn{A}$.
Let $\widetilde{V}$ be the fixed point of $MT_{\widetilde{\Pkern}}$, and $\tilde{\pi} = GT_{\widetilde{\Pkern}}\widetilde{V}$.
Then
\[
	\normb{V^* - V^{\tilde{\pi}}}_{\infty}
		%\leq \frac{2}{1 - \gamma}\normb{\widetilde{V} - MT_{\Pkern} \widetilde{V}}_{\infty}
		\leq \frac{2\gamma}{1 - \gamma}\normb{(\Pkern - \widetilde{\Pkern})\widetilde{V}}_{\infty}.
\]
\end{theorem}

This result is essentially contained in the works of
    \citet[][Corollary to Theorem 3.1]{whitt1978approximations},
    \citet[][Corollary 2]{singh1994upper}\footnote{
    \citet{singh1994upper} correctly bound $\normb{V^* - V^{\tilde{\pi}}}_{\infty}$, but their statement of Corollary 2 suggests that they are
    bounding a different quantity.
    },
    \citet[][Proposition 3.1]{bertsekas2012weighted},
    and \citet[][Lemma 1.1]{Grunewalder2011:arxiv}.

An important implication of this result, which we feel is often overlooked, is that the approximation $\widetilde{\Pkern}$ to $\Pkern$
does not have to be precise everywhere (at all functions $V\in \sn{V}$), but only at $\widetilde{V}$, the fixed point of the approximate model --
a self-fulfilling prophecy, prone to failure?
To understand why this works, consider the case when $\widetilde{\Pkern} \widetilde{V}$ perfectly matches $\Pkern \widetilde{V}$, \ie{}, when the bound on the right-hand side is zero. In this case $\widetilde{V} = MT_{\widetilde{\Pkern}}\widetilde{V} = MT_{\Pkern}\widetilde{V}$, which implies that $\widetilde{V} = V^*$ and, $\tilde{\pi} = GT_{\widetilde{\Pkern}}\widetilde{V} = GT_{\Pkern}V^*$ is optimal.
\emph{The moral is that models do not have to be precise everywhere; if $\Pkern \tilde{V}$ can be estimated, the above inequality can be used to derive a posteriori bounds on the policy error and even form the basis of improving the model.} This can be viewed as a major,
unexpected win for model-based RL.

\Citet{ormoneit2002kernel,barreto2011reinforcement,barreto2011computing,precup2012online,barreto2014policy,barreto2014practical} bound $\norm{V^* - \widetilde{V}}_{\infty}$ rather than the policy error.
We emphasize (\cf{} \cref{sec:valueBoundIssues}) that  $\norm{V^* - \widetilde{V}}_{\infty}$ is not the correct quantity to bound in order to understand the quality of $\pihat$, and that the policy error should be bounded.
As we also discuss in \cref{sec:valueBoundIssues}, this contrasts to ADP bounds, where, in order to understand the policy error in supremum norm, it is sufficient to bound the deviation between the optimal value function and the value estimate that generates the policy.

\subsection{Bounds on the policy error in factored linear models}
Our first novel result is a supremum-norm bound for policy error when we use factored linear models: \cref{thm:inftyNormBound}.
Because we can recover results for unfactored linear models by taking $\proj$ to be the identity mapping over $\sn{X}$, we can use \cref{thm:inftyNormBound} to get a bound that is tighter than \cref{thm:baseline}.
Strictly speaking, taking $\Qkern$ stochastic, $\proj$ as the identity mapping, and upper-bounding the right-hand side of \cref{thm:inftyNormBound} by $2\varepsilon_2$ gives us \cref{thm:baseline}.
\begin{restatable}[Supremum-norm bound]{theorem}{thmInftyNormBound}
\label{thm:inftyNormBound}
Let $\pihat$ be the policy derived from the factored linear model defined using \eqref{eq:pihatdef} and \eqref{eq:ustardef}.
If \cref{ass:boundedQkern,ass:boundedNorm} hold, then
\begin{align}
\label{eq:inftynormBound}
\normb{V^* - V^{\pihat}}_{\infty} \leq \varepsilon(V^*) + \varepsilon(V^{\pihat}),
\end{align}
where $\varepsilon(V) = \min(\varepsilon_1(V),\varepsilon_2)$, and
\begin{align*}
\varepsilon_1(V) &= \gamma \normb{(\Pkern - \Qkern\proj)V}_{\infty} + \frac{B\gamma^2}{1 - \gamma} \normb{\proj(\Pkern - \Qkern\proj)V}_{\infty},\\
\varepsilon_2 &= \frac{\gamma}{1 - \gamma}\mnormb{(\Pkern  - \Qkern\proj)U^*}_{\infty}.
\end{align*}
\end{restatable}

The comments after \cref{thm:baseline} apply to \cref{thm:inftyNormBound}: Curiously, it is enough if the model is ``good'' at its own fixed point.
However, what is most striking about \cref{thm:inftyNormBound} is the $\varepsilon_1(V)$ term.
It means that if $B$ is not too big, and if the error of the model at $V^*$ and $V^{\pihat}$ \emph{in the compressed space $\sn{W}^{\sn{A}}$} is small, then the term that depends on $\frac{1}{1 - \gamma}$ is small.
Moreover, we can expect this term to be easier to control than $\normb{(\Pkern - \Qkern\proj)V}_{\infty}$,
though while the term with $\varepsilon_2$ may lead to \emph{a posteriori} bounds, due to the presence of
$V^*$ and $V^{\pihat}$, objects in the true MDP, the $\varepsilon_1$ terms are better treated as \emph{a priori} bounds.

The proof of \cref{thm:inftyNormBound} (presented below) uses the triangle inequality
\begin{equation}
\norm{V^* - V^{\pihat}} \leq \norm{V^* - U^*} + \norm{U^* - V^{\pihat}}, \label[equation]{eq:triangleIneq}
\end{equation}
combined with \cref{lem:boundOnly} stated next
(\Cref{lem:boundOnly} is a technical lemma and its proof is in \cref{sec:mdpResults}):
\begin{restatable}{lemma}{lemmaBoundOnly}
\label{lem:boundOnly}
Let \cref{ass:LipMaxOperator,ass:boundedQkern} hold, and assume that $\gamma\Lip(\projA \Qkern) \leq \alpha < 1$.
For $V \in \cb{V^*, V^{\pihat}}$ we have that
\begin{equation}
\normb{V - U^*} \leq \gamma\normb{(\Pkern - \Qkern\proj)V} + \frac{B\gamma^2}{1 - \alpha}\normb{\projA (\Pkern - \Qkern\proj)V}.\label[equation]{eq:lem:boundOnly:noContraction}
\end{equation}
Additionally, if $\gamma\Lip(\Pkern) \leq \beta < 1$ (or, alternatively, $\gamma\Lip(M^{\pihat}\Pkern) \leq \beta < 1$), we also have for $V = V^*$ (respectively, $V = V^{\pihat}$) that
\begin{equation}
\normb{V - U^*} \leq \frac{\gamma}{1 - \beta}\mnormb{(\Pkern - \Qkern \proj)U^*}. \label[equation]{eq:lem:boundOnly:contraction}
\end{equation}
\end{restatable}

\Cref{lem:boundOnly}~\eqref{eq:lem:boundOnly:noContraction} can be interpreted as the bound we get by doing a Bellman lookahead with $MT_{\Qkern}$, followed by application of the well-known bound for an $\alpha$-contraction $T$ with fixed point $\widetilde{V}$ \citep{bertsekas1995dynamic}:
\begin{equation}
    \normb{V - \widetilde{V}} \leq \frac{1}{1 - \alpha}\normb{V - TV} \label[equation]{eq:standardFPBound}
\end{equation}
(with $T = M' T_{\projA \Qkern}$ in the case of \cref{lem:boundOnly}).
Similarly, taking $T = MT_{\Pkern}$ ($T = M^{\pihat}T_{\Pkern}$) in \eqref{eq:standardFPBound} combined with $\gamma\Lip(\Pkern) \leq \beta < 1$ ($\gamma\Lip(M^{\pihat}\Pkern) \leq \beta < 1$), allows us to see that $MT_{\Pkern}$ ($M^{\pihat}T_{\Pkern}$) is a $\beta$-contraction, so \eqref{eq:standardFPBound} gives us \cref{lem:boundOnly}~\eqref{eq:lem:boundOnly:contraction} for $V^*$ ($V^{\pihat}$).
\Cref{lem:boundOnly}~\eqref{eq:lem:boundOnly:noContraction} is also interesting in the special case of unfactored linear models (when $\proj$ is the identity mapping) with $\Qkern$ as a non-expansion (\eg{}, $\Qkern$ stochastic): Because $B = 1$ and $\alpha = \gamma$, the bound becomes
\[
\normb{V - U^*} \leq \frac{\gamma}{1 - \gamma}\normb{(\Pkern - \Qkern\proj)V},
\]
and in this case no looseness was introduced by doing a Bellman lookahead and then applying \eqref{eq:standardFPBound}, relative to applying \eqref{eq:standardFPBound} directly.
This will allow us to recover results for unfactored linear models from the bounds we derive from \cref{lem:boundOnly}.

\begin{proof} (of \cref{thm:inftyNormBound})
We will verify the assumptions of \cref{lem:boundOnly}, so that we can bound the terms on the right-hand side (RHS) of \eqref{eq:triangleIneq} with the help of this lemma.
\cref{lem:boundOnly} needs: \cref{ass:LipMaxOperator}, \cref{ass:boundedQkern}, $\gamma \Lip(\projA\Qkern) < 1$, $\gamma \Lip(\Pkern) < 1$ and $\gamma \Lip(M^{\pihat}\Pkern) < 1$.
 \cref{ass:LipMaxOperator} holds by  \cref{prop:LipMaxOperator}, whose 
 assumptions are satisfied because \cref{thm:inftyNormBound} uses supremum norms.
\cref{ass:boundedQkern} holds by assumption.
Next, \cref{ass:boundedNorm} implies that $\gamma \Lip(\projA\Qkern) \leq \gamma < 1$.
Because $\Lip(\Pkern) = 1$ in supremum norm, we get $\gamma \Lip(P) \leq \gamma < 1$.
Finally, $\Lip(\Pkern) = 1$ and \cref{ass:LipMaxOperator} imply together that $\Lip(M^{\pihat}\Pkern) \leq \gamma < 1$.
The result is obtained by using \cref{lem:boundOnly} (with $\alpha = \beta = \gamma$) to bound the terms on the RHS of
\eqref{eq:triangleIneq}.
\end{proof}

\Cref{thm:inftyNormBound} is tight, as shown by \cref{prop:tightness} (\cf{} \cref{sec:mdpResults}).
Trivially, we can use \cref{thm:inftyNormBound} to crudely upper-bound the policy error in $L^p(\mu)$ norm, but the bound we obtain this way is not very interesting.
This is because supremum norm bounds, though easy to prove, can be too harsh: $V^*$ and $V^{\pihat}$ can be close in other meaningful norms, while not being close in supremum norm, in which case the right-hand side of the bound in \cref{thm:inftyNormBound} can be large even if the left-hand side is small (\cf{} \cref{prop:infNormHarsh}, \cref{sec:mdpResults}).

\Citet{de2003linear}
show that the harshness of the supremum norm can be mitigated by considering the policy error in weighted supremum norm.
Intuitively, the error in states that are unlikely to be visited by $\pi^*$ should be underweighted, as we discussed earlier.
Thus, one alternative to supremum norm bounds is to use a generalization of \cref{thm:inftyNormBound} for the weighted supremum norm:
\begin{restatable}[Weighted supremum norm bound]{theorem}{thmWeightedInftyNormBound}
\label{thm:weightedInftyNormBound}
Let $\pihat$ be the policy derived from the factored linear model defined using \eqref{eq:pihatdef} and \eqref{eq:ustardef}.
If \cref{ass:boundedQkern,ass:Lyapunov} hold, then
\[
\normb{V^* - V^{\pihat}}_{\infty, \nu} \leq \varepsilon(V^*) + \varepsilon(V^{\pihat}),
\]
where $\varepsilon(V) = \min(\varepsilon_1(V),\varepsilon_2)$, and
\begin{align*}
\varepsilon_1(V) &= \gamma \normb{(\Pkern - \Qkern\proj)V}_{\infty, \nu} + \frac{B\gamma^2}{1 - \beta_{\eta, \projA\Qkern}} \normb{\proj(\Pkern - \Qkern\proj)V}_{\infty, \eta} \\
\varepsilon_2 &= \frac{\gamma}{1 - \beta_{\nu, \Pkern}}\mnormb{(\Pkern  - \Qkern\proj)U^*}_{\infty, \nu}.
\end{align*}
\end{restatable}

Under \cref{ass:boundedNorm} and \cref{ass:Lyapunov}~\eqref{ass:Lyapunov:nu}, \cref{thm:weightedInftyNormBound} holds with $\beta_{\eta, \projA\Qkern} = \gamma$.
The comments about $\varepsilon_1(V)$ and $\varepsilon_2$ in \cref{thm:baseline,thm:inftyNormBound} are also valid for \cref{thm:weightedInftyNormBound}, but the dependencies are, evidently, expressed in different norms.
%In particular, if $\proj$ is a point evaluator with $\sn{I} \subseteq \sn{S}$, and if $\nu$ does not place much weight on the errors outside of $\sn{I}$, we see that
Moreover, by taking $\nu = x \mapsto 1$ and $\eta = i \mapsto 1$, and by realizing that $\nu$ is $\gamma$-Lyapunov \wrt{} $\Pkern$ and, under \cref{ass:boundedNorm}, $\eta$ is $\gamma$-Lyapunov \wrt{} $\projA\Qkern$, we recover \cref{thm:inftyNormBound} from \cref{thm:weightedInftyNormBound}.
Previously, weighted-supremum norm bounds were derived for ALP. However, the weakness of these bounds is that they are sensitive
to the measure-change between the ``ideal constraint sampling distribution'' (which depends on unknown quantities whose knowledge
basically implies the knowledge of the optimal policy) and the actual one used in the algorithm \citep{de2003linear}.

\begin{proof} (of \cref{thm:weightedInftyNormBound})
We start with the triangle inequality in \eqref{eq:triangleIneq}.
To obtain $\varepsilon_1(V)$ we use \cref{lem:boundOnly}~\eqref{eq:lem:boundOnly:noContraction} with $\alpha = \beta_{\eta, \projA\Qkern}$.
The conditions of \cref{lem:boundOnly}~\eqref{eq:lem:boundOnly:noContraction} are fulfilled by \Cref{prop:LipMaxOperator,ass:boundedQkern}, and because $\eta$ is $\gamma$-Lyapunov \wrt{} $\projA\Qkern$ (via \cref{ass:Lyapunov}~\eqref{ass:Lyapunov:eta}).

\Cref{lem:boundOnly}~\eqref{eq:lem:boundOnly:noContraction} gives $\varepsilon_2$ after we realize that $\Lip(MT_{\Pkern}) \leq \gamma \Lip(\Pkern) = \gamma \beta_{\nu} < 1$ and that $\Lip(M^{\pihat}T_{\Pkern}) \leq \gamma \Lip(\Pkern) = \gamma \beta_{\nu} < 1$, since $\nu$ is $\gamma$-Lyapunov \wrt{} $\Pkern$ by \cref{ass:Lyapunov}~\eqref{ass:Lyapunov:nu}.
\end{proof}

Normally, we are interested in the policy error \wrt{} an initial state distribution, or a stationary distribution of a policy (\eg{}, a stationary distribution of $\pi^*$), and we can naturally consider the policy error in $L^1(\mu)$ norm, where $\mu$ is a measure over $\sn{X}$ that we are interested in.
We can get an immediate bound for the more general $L^p(\mu)$ norm (for any $p \geq 1$) of the policy error, using \cref{thm:weightedInftyNormBound} (\cf{} \cref{thm:muNormBoundInfty}, \cref{sec:mdpResults}).
However, we can also bound the policy error in $L^p(\mu)$ ``directly'', \ie{}, in terms of model errors in $L^p(\mu)$ norm, as \cref{thm:muNormBound}, to be stated next, shows.

In order to state \cref{thm:muNormBound}, we need to use a \emph{concentrability coefficient} $C_{\pihat, \Pkern, \mu, \xi}$ (although part of our bound will be free of this coefficient).
Consider a measure $\xi$ over $\sn{X}$, and the operator $\fun{I - \gamma M^{\pihat}\Pkern}{(\sn{V}, \normb{\,\cdot\,}_{\xi,p})}{(\sn{V}, \normb{\,\cdot\,}_{\mu,p})}$.
If $I - \gamma M^{\pihat}\Pkern$ has no inverse (as an operator acting between the above two spaces), define $C_{\gamma, \pihat, \Pkern, \mu, \xi} \eqdef \infty$, otherwise let the concentrability coefficient be
\begin{equation}
    C_{\gamma, \pihat, \Pkern, \mu, \xi} \eqdef (1 - \gamma)\Lip((I - \gamma M^{\pihat}\Pkern)^{-1}) = (1 - \gamma)\normb{(I - \gamma M^{\pihat}\Pkern)^{-1}}\,. \label[equation]{eq:concentrabilityCoefficient}
\end{equation}
(Note that here both $\Lip(\cdot)$ and $\normb{\cdot}$ hide a dependence on $\xi,\pi$ and $p$.)
As opposed to previous uses of concentrability coefficients \citep{munos2003error,FaMuSze10},
our coefficient depends only on the policy computed, which makes it more suitable for the estimation of our bound.
In case the $C_{\gamma, \pihat, \Pkern, \mu, \xi}$ is not very large, we can get meaningful bounds from \cref{thm:muNormBound} from $\varepsilon_2$, but even if $C_{\gamma, \pihat, \Pkern, \mu, \xi} = \infty$ and $\varepsilon_2$ is vacuous, we can still get a priori bounds with a dependence on $\varepsilon_1(V^{\pihat})$, in addition to the dependence on $\varepsilon_1(V^*)$.
The $\varepsilon_1(V)$ term can be analyzed similarly to its analogues in \cref{thm:inftyNormBound,thm:weightedInftyNormBound}, modulo the norm differences.
We are flexible about the choice of $\norm{\cdot}_{\sn{W}}$ (which nonetheless affects \cref{ass:boundedQkern,ass:boundedNorm}).
One may think of choosing $\norm{\cdot}_{\sn{W}} = \norm{\cdot}_{\rho,p}$ for some $\rho$, however with this norm choice, \cref{ass:boundedNorm} becomes restrictive.
When it comes to satisfying \cref{ass:boundedNorm}, a weighted supremum norm is reasonable, as discussed earlier, so we choose this norm as the norm over the compressed space $\sn{W}$ in \cref{thm:muNormBound}.
\begin{restatable}[$L^p(\mu)$ norm bound]{theorem}{thmMuNormBound}
\label{thm:muNormBound}
Let $\pihat$ be the policy derived from the factored linear model defined using \eqref{eq:pihatdef} and \eqref{eq:ustardef}.
Choose the norms so that 
$\norm{\cdot}_{\sn{V}} = \norm{\, \cdot \,}_{\mu,p}$ and $\norm{\,\cdot\,}_{\sn{W}} = \norm{\,\cdot\,}_{\infty, \eta}$.
If \cref{ass:boundedQkern,ass:boundedNorm} hold, then
\als{
\normb{V^* - V^{\pihat}}_{\mu,p} \leq \varepsilon_1(V^*) + \min\cb{\varepsilon_1(V^{\pihat}), \varepsilon_2},
}
where
\als{
	&\varepsilon_1(V) = \gamma\normb{(\Pkern - \Qkern\proj)V}_{\mu,p} + \frac{B \gamma^2}{1 - \gamma}\normb{\projA (\Pkern - \Qkern\proj)V}_{\infty, \eta},\\
	&\varepsilon_2 = C_{\gamma, \pihat, \Pkern, \mu, \xi}\,\frac{\gamma}{1 - \gamma}\normb{(\Pkern - \Qkern \proj)U^*}_{\xi,p},
}
where $C_{\gamma, \pihat, \Pkern, \mu, \xi}$ is defined in \eqref{eq:concentrabilityCoefficient}.
\end{restatable}
Before the proof, let us point out that $\varepsilon_1$ is independent of the concentrability coefficient.
Further, as remarked beforehand,
its dependence on the discount factor can be quite mild (if the second term in the definition of $\varepsilon_1$ is small).
\begin{proof}
The first step is to use \eqref{eq:triangleIneq}.
Then we see that \Cref{prop:LipMaxOperator} ensures that \cref{ass:LipMaxOperator} is satisfied, and \cref{ass:boundedNorm} guarantees that $\normb{\projA\Qkern} \leq 1$.
Thus, \cref{lem:boundOnly}~\eqref{eq:lem:boundOnly:noContraction} with $\alpha = \gamma$ gives us $\normb{U^* - V}_{\mu,p} \leq  \varepsilon_1(V)$ for $V \in \cb{V^*, V^{\pihat}}$.

To bound $\normb{U^* - V^{\pihat}}_{\mu,p} \leq \varepsilon_2(V^{\pihat})$ we proceed as follows.
If $(I - \gamma M^{\pihat}\Pkern)$ is not invertible, then $C_{\gamma, \pihat, \Pkern, \mu, \xi} = \infty$ and the result holds vacuously, so assume otherwise.
Since $V^{\pihat} = M^{\pihat}T_{\Pkern}V^{\pihat}$,
\[
    (I - \gamma M^{\pihat}\Pkern)V^{\pihat} = M^{\pihat}r.
\]
Moreover,
\[
    U^* - \gamma M^{\pihat}\Pkern U^* - M^{\pihat}r = U^* - M^{\pihat}T_{\Pkern}U^*.
\]
Hence,
\als{
\normb{U^* - V^{\pihat}}_{\mu,p}
&= \normb{(I - \gamma M^{\pihat}\Pkern)^{-1}(I - \gamma M^{\pihat}\Pkern)(U^* - V^{\pihat})}_{\mu,p}\\
&\leq \Lip((I - \gamma M^{\pihat}\Pkern)^{-1})\, \normb{(I - \gamma M^{\pihat}\Pkern)(U^* - V^{\pihat})}_{\xi,p}\\
&= C_{\gamma, \pihat, \Pkern, \mu, \xi} \, \frac{1}{1 - \gamma}\normb{U^* - M^{\pihat}T_{\Pkern}U^*}_{\xi,p} \\
&\leq C_{\gamma, \pihat, \Pkern, \mu, \xi} \, \Lip(M^{\pihat})\frac{\gamma}{1 - \gamma} \normb{(\Pkern - \Qkern \proj)U^*}_{\xi,p}
}
To conclude, we use that $\Lip(M^{\pihat}) \leq 1$ by \cref{prop:LipMaxOperator}.
\end{proof}

\section{Discussion and summary}
\label{sec:related}

Our results unify, strengthen and extend previous works.
The unifying framework of factored linear models was introduced by \citet{yao2014pseudo}.
The focus of the present work is the derivation of policy error bounds, while putting issues of designing and analyzing algorithms to learn models aside.
We believe that in fact this should be the preferred approach to developing theories for reinforcement learning: By first figuring out
what quantities control the policy error in a given error, one is in a better position to design learning algorithms which then control the said
quantities (this is distantly reminiscent to choosing surrogate losses in supervised learning).

Previous work that derives policy error bounds goes back to at least \citet{whitt1978approximations}. In fact,
looking at the literature we see that the results of \citet{whitt1978approximations} have been independently re-derived in part or as a whole multiple times (often confounded with the issue of statistical questions), e.g., in the works mentioned in \cref{sec:approach}.
Compared to the work of \citet{whitt1978approximations}, main advances in deriving policy error bounds have been the introduction of norms
other than the supremum norm, though this happened in different contexts (e.g., \citealt{de2003linear,munos2003error}),
and breaking down the bound of \citet{whitt1978approximations} to more specialized models
(e.g., \citealt{ormoneit2002kernel,ravindran2004algebraic,barreto2011reinforcement,sorg2009transfer}).

One of the main novelties of the present work is that we are importing previous techniques
to model-based RL to obtain policy error bounds in norms other than (unweighted) supremum norms.
In particular,
	to derive policy error bounds that use weighted supremum norms, we are building on
	the work of \citet{de2003linear}, and we bring Lyapunov analysis from the approximate linear programming (ALP) methodology to model-based RL.
	At the same time, to derive  policy error bounds that use weighted $L^p$-norms we import ideas from \citet{munos2003error}, who analyzed
	approximate dynamic programming (ADP) algorithms.
    During this process we streamlined the definitions from these works by  sticking to the language of operator algebras (specifically, Banach lattices).
The use of this language has two main benefits:
It allowed us to present shorter and rather direct proofs, while it also shed light on the algebraic and geometric assumptions that were key in the proofs.
We believe that our operator algebra approach could also improve previous results in either ALP or ADP.
An interesting avenue for further work is to investigate the minimum set of assumptions under which our calculations remain valid:
At present it appears that we use very little of the rich structure of the function spaces involved.
We speculate that the results can also be proven in certain max-plus (a.k.a.~tropical) algebras, leading to results that may hold, \eg{}, for various
versions of sequential games.

Another major novel aspect of the present work is that we tightened previous bounds. In particular, our bounds come in two forms:
One form (the ``$\varepsilon_1$'' term) tells us how model errors should be controlled in the spaces of compressed value functions,
while the other form (the ``$\varepsilon_2$'' term) tells us that it is enough if the model operator approximates the true model operator at
only the (uncompressed) value function derived \emph{from the model}.

While we shorten and improve previous results, we also managed to relax the key condition of previous works that required that the Bellman operator acting on uncompressed value functions and underlying the model needs to be a contraction.
While we are still relying on contraction-type arguments, the contraction arguments are used with the compressed space, as previously suggested (but not analyzed) by  \citet{yao2014pseudo}. We feel that it is more natural to require that the Bellman operator for the compressed space is a contraction than to require the same for the respective operator acting on the uncompressed space. Indeed, our bounds show that this second assumption is entirely superfluous (cf. the ``$\varepsilon_2$'' terms).

One limitation of the results presented so far is that we assumed that $\proj$ was a join-homomorphism.
In many models, such as state-aggregation (soft or not) or stochastic factorization
\citet{VanRoy06:Aggregation,barreto2011reinforcement}, $\proj$ is linear (and stochastic) 
but is \emph{not} a join-homomorphism.
Investigating our proofs reveals that we can allow $\proj$ to be a linear operator (and $\projA$ to be a linear operator \st{} $(\projA)^a \neq \proj$) at the price of introducing additional error terms.
For the sake of illustration, in \Cref{thm:muNormBoundStochastic} below
we present a version of the $L^p(\mu)$-norm bounds (and a sketch of proof) that can be obtained for such operators.

For presenting \cref{thm:muNormBoundStochastic}, we will use the greedy action selector in the compressed space as well, \ie{} $G'$ mapping compressed action value functions to policies in $\sn{W}$ (\ie{}, $M^{\prime G' w}w = M' w$ for $w \in \sn{W}^{\sn{A}}$).
It is important to recall the definition of $U^*$ for \cref{thm:muNormBoundStochastic}: $U^* = MT_{\Qkern}u^*$.
Note that if $\proj M = M' \projA$, then we also have $U^* = MT_{\Qkern\proj}U^*$, and we can recover \cref{thm:muNormBound} from \cref{thm:muNormBoundStochastic}. However, $U^*$ is not a fixed point of $ MT_{\Qkern\proj}$ in general when $\proj$ is not a join-homomorphism, a fact that will be important in our discussion below.
\begin{restatable}[$L^p(\mu)$ norm bound for linear $\proj, \projA$]{theorem}{thmMuNormBoundStochastic}
\label{thm:muNormBoundStochastic}
Let $\pihat$ be the policy derived from the factored linear model defined using \eqref{eq:pihatdef} and \eqref{eq:ustardef}.
Choose the norms so that $\norm{\, \cdot \,}_{\sn{V}} = \norm{\, \cdot \,}_{\mu,p}$ and $\norm{\, \cdot \,}_{\sn{W}} = \norm{\, \cdot \,}_{\infty, \eta}$.
Assume that $\proj, \projA$ are linear (but not necessarily join-homomorphisms, and $(\projA)^a$ not necessarily equal to $\proj$).
If \cref{ass:boundedQkern,ass:boundedNorm} hold, then
\begin{equation}
\normb{V^* - V^{\pihat}}_{\mu,p} \leq \varepsilon_1(V^*, M') + \min\cb{\varepsilon_1(V^{\pihat}, M^{\prime G' T_{\projA\Qkern}u^*}), \varepsilon_2}, \label[equation]{eq:muNormBoundStochastic}
\end{equation}
where
\als{
    \varepsilon_1(V, N') &= \gamma\normb{(\Pkern - \Qkern\proj)V}_{\mu,p} \\
    &\phantom{=}~+\frac{B \gamma}{1 - \gamma}\pb{\normb{\proj V - N'\projA T_{\Pkern}V}_{\infty,\eta} + \gamma \normb{\projA (\Pkern - \Qkern\proj)V}_{\infty,\eta} },\\
    \varepsilon_2 &= C_{\gamma, \pihat, \Pkern, \mu, \xi}\,\frac{1}{1 - \gamma}\normb{\Pkern U^* - \Qkern u^*}_{\xi,p},
}
and where $C_{\gamma, \pihat, \Pkern, \mu, \xi}$ is defined in \eqref{eq:concentrabilityCoefficient}.
\end{restatable}
\begin{proof}(Sketch)
The $\varepsilon_1$ terms are obtained by appropriately modifying \cref{lem:QKernFPBound} (which is an intermediate result, presented in the appendix, that is used in the proof of \cref{lem:boundOnly}), as we describe below.
We will take $V = V^*$ ($V = V^{\pihat}$), $N = M$ (resp.~$N = M^{\pihat}$) and $N' = M'$ (resp.~$N' = M^{\prime G' T_{\projA\Qkern}u^*}$.
Then the identity $u^* = N' T_{\projA\Qkern}u^*$ holds.

Because we cannot use the identity $\proj M = M' \projA$, we need to use the following chain of inequalities:
\als{
\normb{\proj V - u^*}
&= \inf_{k \geq  1}\frac{1}{1 - \alpha^{k}} \mnormb{\proj V - (N' T_{\projA \Qkern})^k \proj V}\\
&\leq \frac{1}{1 - \alpha} \mnormb{\proj NT_{\Pkern}V - N' \projA T_{\Qkern\proj} V} \\
&\leq \frac{1}{1 - \alpha} \pb{ \mnormb{\proj NT_{\Pkern}V - N' \projA T_{\Pkern}V} + \mnormb{N' \projA T_{\Pkern}V - N' \projA T_{\Qkern\proj} V} } \\
&\leq \frac{1}{1 - \alpha} \pb{ \mnormb{\proj NT_{\Pkern}V - N' \projA T_{\Pkern}V} + \gamma \mnormb{\projA (\Pkern - \Qkern\proj)V} }.
}

To obtain $\varepsilon_2$, we cannot use that $U^* = MT_{\Qkern\proj}U^*$, so we simply write
\als{
\normb{U^* - M^{\pihat}T_{\Pkern} U^*}
&= \normb{M^{\pihat}T_{\Qkern}u^* - M^{\pihat}T_{\Pkern} U^*} \\
&\leq \normb{\Qkern u^* - \Pkern U^*}.
}
The above can be used to modify \cref{lem:boundOnly} as well, leading to analogues of \cref{thm:inftyNormBound,thm:weightedInftyNormBound} where $\proj$ is linear but not a join-homomorphism.
\end{proof}
Note that both this result and  \cref{thm:muNormBound} show a curious scaling as a function of $1/(1-\gamma)$.
In fact, the astute reader may recall that policy error bounds typically scale with $1/(1-\gamma)^2$.
A little thinking reveals that our result may be subject to the same scaling:
Just like in \cref{thm:baseline}, where $\widetilde{V}$ hides $1/(1-\gamma)$, 
in the above bounds the value functions themselves bring in another $1/(1-\gamma)$, too.
Is the scaling with $1/(1-\gamma)^2$ necessary? 
The answer is no:
Theorem 4.1 of \citet{VanRoy06:Aggregation} shows 
that in some version of state-aggregation the policy error can scale with $1/(1-\gamma)$ only
 (as a side-note, the only result so far with this property).
Thus, it may be worthwhile to look at the differences between Theorem 4.1 and the above result.
First, recall that in his Theorem 4.1 \citet{VanRoy06:Aggregation}  bounds the error of the policy $\tilde{\pi}$ that is greedy 
with respect to the fixed point $\tilde{U}^*$ of $MT_{\Qkern\proj}$, where $\proj = \proj_{\tilde{\pi}}$ is chosen to depend on the policy (for some policy $\pi$, $\proj_{\pi}$ is a weighted Euclidean projection to the compressed space induced by the aggregation, where the weights depend on the stationary distribution of ${\pi}$).
Formally, the policy  is defined by $\tilde{\pi} = GT_{\Qkern\proj_{\tilde{\pi}}} \tilde{U}^*$ where $\tilde{U}^* = MT_{\Qkern\proj_{\tilde{\pi}}} \tilde{U}^*$.
Thus, the policy whose error he bounds is different from ours in two respects:
As pointed out above, $U^* = MT_{\Qkern\proj} u^*$ (that our $\pihat$ is greedy with respect to) 
is not necessarily the fixed point of $MT_{\Qkern\proj}$. Further, our result is proven for general $\proj$.
At this time it is not clear whether with a specific choice of $\proj$ (like $\proj_{\hat{\pi}}$) the terms involved in the definition of $\epsilon_1$ would cancel
the additional $1/(1-\gamma)$ factor.
For what it is worth, we note that for the ``counterexample'' that \citet{VanRoy06:Aggregation} presents,
when $\proj = \proj_{\hat{\pi}}$, $\epsilon_1$ scales with $1/(1-\gamma)$ only (as opposed to scaling with $1/(1-\gamma)^2$),
showing that our bound has the ability to exploit the benefits of a ``good'' choice of $\proj$.
However, it remains to be seen whether this or some other systematic way of choosing $\proj$ will always cancel
the extra $1/(1-\gamma)$ factor.

To summarize, this paper advances our understanding of model errors on policy error in reinforcement learning.
We do this by improving previous bounds by using a versatile set of norms and introducing a completely new bound which has the potential of better scaling with the discount factor, while at the same time we extend the range of the models by relaxing previous assumptions.
We also showed that (some) of our bounds are unimprovable.
By effectively using the language of Banach lattices, our proofs are shorter, while at the same time hold the promise of being generalizable beyond MDPs.
We believe that our approach may lead to advances in the analysis and design of alternate approaches to reinforcement learning, namely both in approximate linear programming and approximate dynamic programming.

\section*{Acknowledgments}

This work was supported by Alberta Innovates Technology Futures and NSERC.

\bibliography{bibliography}

\appendix

\newpage
\section{List of operators}
\label{sec:listop}
For ease of reference, we present \cref{tab:operators}, which gives a summary of the operators we define and use.

\begin{table}[h]
\centering
\begin{tabular}{l l l}
\toprule
Operator 		& Between						& Definition \\
\midrule
$\Pkern^a$ 	& $\sn{V} \to \sn{V}$ 		& $(\Pkern^a V)(x) \eqdef \int V(x') \Pkern^a(dx'|x)$ \\
$\Qkern^{a}$ & $\sn{W} \to \sn{V}$ 		& $(\Qkern^{a}w)_i \eqdef \int w_{i'} \Qkern^{a}(di' | i)$ \\
$\Pkern$ 		& $\sn{V} \to \sn{V}^{\sn{A}}$ 			& $(\Pkern V)(a) (= (\Pkern V)^a ) \eqdef  \Pkern^a V$ \\
$\Qkern$ 		& $\sn{W} \to \sn{V}^{\sn{A}}$ 			& $(\Qkern V)(a) (= (\Qkern V)^a ) \eqdef  \Qkern^a V$ \\
$\proj$	& $\sn{V} \to \sn{W}$ 	& almost always a join-homomorphism \\%$(\proj V)_i = V(x_i)$ \\
$\projA$	& $\sn{V}^{\sn{A}} \to \sn{W}^{\sn{A}}$	& $(\projA V)^a = \proj V,~ \forall a \in \sn{A}$  \\
$\projA \Qkern$ 	& $\sn{W} \to \sn{W}^{\sn{A}}$ 		& $(\projA\Qkern w)(a) (= (\projA\Qkern w)^a ) \eqdef  \proj\Qkern^a w$ \\
$M$ 	& $\sn{V}^{\sn{A}} \to \sn{V}$ 	& $(MV)(x) \eqdef \max_a V^a(x)$ \\
$M'$	& $\sn{W}^{\sn{A}} \to \sn{W}$	& $(M' w)_i \eqdef \max_a w^a_i$ \\
$M^{\pi}$ 	& $\sn{V}^{\sn{A}} \to \sn{V}$ 	& $(M^{\pi}V)(x) \eqdef V^{\pi(x)}(x)$ \\
$M^{\prime\pi}$ 	& $\sn{W}^{\sn{A}} \to \sn{W}$ 	& $(M^{\prime\pi}w)_i \eqdef w^{\pi(i)}_i$ \\
$G$ 	& $\sn{V}^{\sn{A}} \to \Pi$	 	& $GV(x) \eqdef \argmax_a V^a(x)$ \\
$T_{\Pkern}$ 	& $\sn{V} \to \sn{V}^{\sn{A}}$	 	& $T_{\Pkern}V \eqdef  r + \gamma \Pkern V$ \\
$T_{\Qkern}$ 	& $\sn{W} \to \sn{V}^{\sn{A}}$	 	& $T_{\Qkern}w \eqdef  r + \gamma \Qkern w$ \\
$T_{\Qkern\proj}$ 	& $\sn{V} \to \sn{V}^{\sn{A}}$	 	& $T_{\Qkern\proj}V \eqdef  T_{\Qkern}\proj V = r + \gamma \Qkern \proj V$ \\
$T_{\projA\Qkern}$ 	& $\sn{W} \to \sn{W}^{\sn{A}}$	 	& $T_{\projA\Qkern}w \eqdef \projA T_{\Qkern} w= \projA r + \gamma \projA\Qkern w$ \\
\bottomrule
\end{tabular}
\caption{Definitions of operators used in the paper.
\label{tab:operators}}
\end{table}

\section{General results}
\label{sec:generalResults}

In this section, we present some general technical results.

\begin{proposition}
\label{prop:uniqueFixedPoint}
Consider an operator $\fun{T}{\sn{V}}{\sn{V}}$ mapping a normed space $(\sn{V},\norm{\,\cdot\,}_{\sn{V}})$ to itself.
If $\Lip(T) < \infty$ and $T^m$ is a contraction for some $m > 0$, then $T$ has a unique fixed point.
\end{proposition}
\begin{proof}
Banach's fixed point theorem ensures that $T^m$ has a unique fixed point $V$, which must also be the unique fixed point of $T^{m^2}$ and $T^{m(m + 1)}$, so $V = T^{m(m + 1)} = TT^{m^2}V = TV$, so $V$ is a fixed point of $T$.
Since every fixed point of $T$ is also a fixed point of $T^m$, it follows that $V$ is the unique fixed point of $T$.
\end{proof}

\lemBContractionFixedPoint*
\begin{proof}
Since $(HT)^m$ is a contraction, \cref{prop:uniqueFixedPoint} ensures that $(HT)^m$ has a unique fixed point $W^*$, which is also the unique fixed point of $HT$.
Defining $V^* \eqdef TW^*$, we can see that $THV^* = THTW^* = TW^* = V^*$.
It remains to show that $V^*$ is the unique fixed point of $TH$, so let us assume that there exists $V' \neq V^*$ \st{} $V' = THV'$.
Then with $W' \eqdef HV'$ we have $TW' = V'$.
Now, $HTW' = HV' = W'$, so $W'$ is a fixed point of $HT$, which implies $W' = W^*$, since the fixed point of $HT$ is unique, but then $V' = TW' = TW^* = V^*$, which is a contradiction.
\end{proof}
\begin{lemma}
\label{lem:bContractionBound}
Let $(\sn{V}, \norm{\,\cdot\,})$ be a Banach space and $\fun{T}{\sn{V}}{\sn{V}}$ be an operator.
Assume that there exists $V^* \in \sn{V}$ such that $TV^* = V^*$, and that there exist constants $a < 1$ and $b$ such that for all  $m \geq 0$ we have $\Lip(T^{m+1}) \leq ba^m$.
Then for all $V \in \sn{V}$ and $m\ge 0$ such that $ba^m<1$,
\[
	\normb{V - V^*} \le \frac{1}{1 - ba^m} \normb{V - T^{m+1} V}.
\]
Further, if we take the infimum of both sides for $m$ such that $ba^m<1$, we get an equality.
\end{lemma}
\begin{proof}
We have that for all $m \geq 0$,
\als{
	\normb{V - V^*} &= \normb{V - T^{m+1}V^*}\\
	&= \normb{V - T^{m+1}V  + T^{m+1}V - T^{m+1}V^*}\\
	&\leq \normb{V - T^{m+1}V} + \normb{T^{m+1}V - T^{m+1}V^*}\\
	&\leq \normb{V - T^{m+1}V} + ba^m \normb{V - V^*}.
}
To arrive at an upper-bound, we need to move the third term to the right-hand side and divide the inequality by $1 - ba^m$.
The inequality is preserved after division only for those $m$ when $ba^m < 1$, giving the result.

To see why we get the equality, note that $T^{\infty}V = V^*$. Hence,
\[
	\inf_{m : ba^m < 1}\frac{1}{1 - ba^m} \normb{V - T^{m+1} V} \leq \normb{V - V^*}.
\]
\end{proof}
\propOnlyLJH*
\begin{proof}
Consider $v \geq 0$.
We can write $v = \sum_{j=1}^m v_j e_j$, where $\pb{e_j}_{j=1}^m$ is the Euclidean basis.
Because $v \geq 0$, we can also write $v = \bigvee_{j=1}^m v_j e_j$.
By linearity of $\proj$, we have that $\proj v = \sum_{j=1}^m v_j \proj e_j$, and since $\proj$ is a join-homomorphism and linear, we also have $\proj v = \bigvee_{j=1}^m \proj (v_j e_j) = \bigvee_{j=1}^m v_j \proj e_j$.

Next, we show that for all $i$, $(\proj e_j)_i \neq 0$ for at most one $j \in [m]$.
Taking $v$ \st{} $v_i = 1$ for all $i$, we have that for all $i \in [n]$
\[
    \pb{\sum_{j=1}^m \proj e_j}_i = (\proj v)_i = \pb{\bigvee_{j=1}^m \proj e_j}_i,
\]
which implies that for all $i \in [n]$ there is at most one $j \in [m]$ \st{} $(\proj e_j)_i > 0$, and $J_i$ is defined as such $j$ if it exists, otherwise arbitrary.
Defining $a_i \eqdef (\proj e_{J_i})_i$ for ($i \in [n]$) gives the result.
\end{proof}
\propLyapunov*
\begin{proof}
Define $A(U) \eqdef \cb{U' \in \sn{U} : \ab{U'} = \ab{U}}$ for $U \in \sn{U}$.
Since $J$ is linear, $\Lip(J) =\normb{J}$. 
Since $\norm{\cdot} \eqdef \norm{\cdot}_{\infty,\nu}$ is a type of supremum norm,
$\normb{J} =  \max_a \normb{J^a}$
(the maximum over the actions and states commute).
Thus, we have that
\als{
\Lip(J)
&= \sup_{U : \norm{U} = 1} \max_a \norm{J^a U}\\
&= \sup_{U : \norm{U} = 1} \max_a \sup_{x} \frac{\ab{(J^a U)(x)}}{\nu(x)}\\
&= \sup_{x}\max_a \sup_{U > 0: \norm{U} = 1}\sup_{U' \in A(U)}  \frac{\ab{(J^a U')(x)}}{\nu(x)}.
}
Note that equality still holds in the last line by equivalence of the suprema with the supremum on the previous line.
The term $\sup_{U' \in A(U)}\frac{\ab{(J^a U')(x)}}{\nu(x)}$ can be maximized \wrt{} $U$ by maximizing $U$ subject to $\frac{U(x)}{\nu(x)} \leq 1$, for all $x \in \sn{X}, a \in \sn{A}$.
Therefore the term is maximized by $U = \nu$, and, since $A(\nu) = \cb{U' \in \sn{U} : \ab{U'} = \nu}$, we get
\[
\gamma \Lip(J) = \gamma \sup_{U : \ab{U} = \nu}\normb{J U} = \beta_{\nu, J}.
\]
\end{proof}

\section{MDP-specific results}
\label{sec:mdpResults}

In this section, we present accessory results and proofs omitted from the main text.
\Cref{lem:QKernFPBound} is an intermediate result for \cref{lem:boundOnly}.
The proof of \cref{lem:boundOnly} is also presented here.
Moreover, we present the proof of three omitted results: \cref{prop:tightness,prop:infNormHarsh,thm:muNormBoundInfty}, respectively a tightness example for \cref{thm:inftyNormBound}, an example showing that the \cref{thm:inftyNormBound} can be harsh, and a weighted supremum norm bound for the policy error in $L^p(\mu)$ norm.

\begin{lemma}
\label{lem:QKernFPBound}
Let \cref{ass:LipMaxOperator,ass:boundedQkern} hold, and assume that $\gamma\Lip(\projA \Qkern) \leq \alpha < 1$.
Then, for $V \in \cb{V^*, V^{\pihat}}$
\[
\mnormb{V - U^*} \leq \gamma\normb{(\Pkern - \Qkern\proj)V} + \frac{B\gamma^2}{1 - \alpha}\normb{\projA (\Pkern - \Qkern\proj)V},
\]
\end{lemma}
\begin{proof}
Using that $V^* = MT_{\Pkern} V^*$, $V^{\pihat} = M^{\pihat}T_{\Pkern}V^{\pihat}$ and $M^{\pihat}T_{\Qkern\proj}U^* = U^* = MT_{\Qkern\proj}U^*$, we first upper-bound, with $N = M$ ($N = M^{\pihat}$) and $V = V^*$ ($V = V^{\pihat}$),
\als{
\normb{V - U^*}
&= \normb{NT_{\Pkern}V - NT_{\Qkern\proj}U^*} \\
&\leq \normb{NT_{\Pkern}V - NT_{\Qkern\proj}V} + \normb{NT_{\Qkern\proj}V - NT_{\Qkern\proj}U^*} \\
&\leq \gamma \Lip(N)\normb{(\Pkern - \Qkern\proj)V} + \Lip(NT_{\Qkern}) \normb{\proj(V - U^*)} \\
&\leq \gamma \normb{(\Pkern - \Qkern\proj)V} + B\gamma \, \normb{\proj(V - U^*)}.
}
Given $N \in \cb{M, M^{\pihat}}$, we define $N'$ as the operator satisfying $\proj N = N' \projA$.
In particular, if $N = M$, then $N' = M'$, otherwise $N' = M^{\prime \pi_2}$ for some $\fun{\pi_2}{\sn{I}}{\sn{A}}$ (in either case $N'$ is well-defined because $\proj$ is a join-homomorphism, \cf{} \cref{ass:joinHomomorphism}).
Since $\Lip(N') \leq 1$ by \cref{ass:LipMaxOperator}, we get $\Lip(N' T_{\projA\Qkern}) \leq \gamma\Lip(\projA \Qkern) \leq \alpha < 1$.
\Cref{lem:bContractionBound} with $T = N' T_{\projA\Qkern}$ and $a = b = \alpha < 1$, combined gives for $N = M$ ($N = M^{\pihat}$) and $V = V^*$ ($V = V^{\pihat}$),
\als{
\normb{\proj (V - U^*)}
&= \inf_{k \geq  1}\frac{1}{1 - \alpha^{k}} \mnormb{\proj V - (N' T_{\projA \Qkern})^k \proj V}\\
&\leq \frac{1}{1 - \alpha} \mnormb{N' \projA T_{\Pkern}V - N' \projA T_{\Qkern\proj} V} \\
&\leq \frac{\gamma}{1 - \alpha} \Lip(N') \mnormb{\projA (\Pkern - \Qkern\proj) V} \\
&\leq \frac{\gamma}{1 - \alpha} \mnormb{\projA (\Pkern - \Qkern\proj) V},
}
where we have also used that $V = NT_{\Pkern} V$ and that $\proj N = N' \projA$.
Combining the above gives,
\als{
\normb{V - U^*}
&\leq \gamma \normb{(\Pkern - \Qkern\proj)V} + B\gamma \normb{\proj(V - U^*)}\\
&\leq \gamma\normb{(\Pkern - \Qkern\proj)V} + \frac{B\gamma^2}{1 - \alpha}\normb{\projA (\Pkern - \Qkern\proj)V}.
}
\end{proof}
\lemmaBoundOnly*
\begin{proof}
Recall that $V^*$ is the optimal value function, \ie{}, the fixed point of the Bellman optimality equation $V^* = MT_{\Pkern}V^*$.
Recall also that $V^{\pihat}$ is the value function of $\pihat$ and the fixed point of the Bellman equation $V^{\pihat} = M^{\pihat}T_{\Pkern}V^{\pihat}$.
%We have seen in \cref{lem:QKernMain} that $U^*$ is the (unique) fixed point of $MT_{\Qkern\proj}$.
\Cref{lem:QKernFPBound} gives us \eqref{eq:lem:boundOnly:noContraction} directly.

To prove \eqref{eq:lem:boundOnly:contraction} for $V = V^*$, we use \cref{lem:bContractionBound} with $T = MT_{\Pkern}$ and $a = b = \beta$, which gives
\als{
\normb{V^* - U^*}
&= \inf_{k \geq 1}\frac{1}{1 - \beta^k}\normb{U^* - (MT_{\Pkern})^k U^*} \\
&\leq \frac{1}{1 - \beta}\normb{MT_{\Qkern\proj}U^* - MT_{\Pkern} U^*}\\
&\leq \frac{\gamma}{1 - \beta}\Lip(M)\normb{(\Pkern - \Qkern\proj) U^*},
}
and then we plug in $\Lip(M) \leq 1$.
For \eqref{eq:lem:boundOnly:contraction} for $V = V^{\pihat}$, we observe that $\Lip(M^{\pihat}T_{\Pkern}) = \gamma\Lip(M^{\pihat}\Pkern)$, then we follow a similar approach:
\als{
\normb{V^{\pihat} - U^*}
&= \inf_{k \geq 1}\frac{1}{1 - \beta^k}\normb{U^* - (M^{\pihat}T_{\Pkern})^k U^*} \\
&\leq \frac{1}{1 - \beta}\normb{M^{\pihat}T_{\Qkern\proj}U^* - M^{\pihat}T_{\Pkern} U^*}\\
&\leq \frac{\gamma}{1 - \beta}\Lip(M^{\pihat})\normb{(\Pkern - \Qkern\proj) U^*},
}
and plug in $\Lip(M^{\pihat}) \leq 1$.
\end{proof}

The recipe for constructing the example proving \cref{prop:tightness} is simple: i) create a three-state, two-action MDP with a ``fork'' state $s_1$ leading to a high-value terminal state $s_2$ with action $a_1$ and a low-value terminal state $s_3$ with action $a_2$; ii) choose the rewards so that the immediate reward $r^{a_2}(s_1) > r^{a_1}(s_1)$, while the value of $(s_1, a_1)$ is higher than $(s_1, a_2)$; iii) make a poor model for the fork state $s_1$, so that $\pihat$ becomes nearsighted, picking $a_2$ rather than $a_1$.
We can also perturb the model for $s_3$, in order to have a desired value for $\norm{V^* - U^*}_{\infty}$.
The rest of the effort pertains to choosing the rewards and the model carefully in order to have the correct value for $\normb{(\Pkern - \Qkern\proj) U^*}_{\infty}$.
There is factor of $\frac{1}{\gamma}$ in the scaling of the rewards, as a result of requiring the return from $s_1$ after the first action to dominate the immediate reward at $s_1$, and the rewards also scale with $\max{\tau}$ for the bound to scale.
The example underlying \cref{prop:tightness} is also well-defined for $\varepsilon = 0$, but then $GT_{\projA \Qkern}u^*$ is no longer unique: It can yield an optimal policy or a policy that is $\tau$-suboptimal in $\normb{\,\cdot\,}_{\infty}$, depending on how ties are broken.
\begin{restatable}[\Cref{thm:inftyNormBound} is tight]{proposition}{propTightness}
\label{prop:tightness}
There exist $\Pkern, \Qkern$ and $\proj$ \st{} for every $\gamma \in (0, 1)$, $\tau \geq 0$ and $\varepsilon \in (0, 1)$ there exists $r \in \sn{V}^{\sn{A}}$ (the rewards scale with $\frac{1 - \gamma^2}{\gamma}\tau$) \st{} $\Lip(\Qkern\proj) < \infty$, $\Lip(\projA\Qkern) \leq 1$, $\frac{2\gamma}{1 - \gamma}\normb{(\Pkern - \Qkern\proj) U^*}_{\infty} = \tau$, and
$\normb{V^* - V^{\pihat}}_{\infty} = (1 - \varepsilon)\tau$.
Thus, \cref{thm:inftyNormBound} can be made arbitrarily tight.
\end{restatable}
\begin{proof}
The set of states is $\sn{X} = \cb{x_1,\ldots,x_3}$, the set of actions is $\sn{A} = \cb{a_1, a_2}$ and the transition probability kernel is specified by $\Pkern$ as follows:
\als{
\Pkern^{a_1} = \left( \begin{array}{c c c}
0 & 1 & 0\\
0 & 1 & 0\\
0 & 0 & 1
\end{array}\right), & & \Pkern^{a_2} = \left( \begin{array}{c c c}
0 & 0 & 1\\
0 & 1 & 0\\
0 & 0 & 1
\end{array}\right)
}
We let $\sn{W} = \rl{2}$ and $\proj V \eqdef (V(x_2), V(x_3))\tra$.

The model is
\als{
\Qkern^{a_1} = \left( \begin{array}{c c}
-1 & 0\\
0 & 1\\
1 & 0
\end{array}\right), & & \Qkern^{a_2} = \left( \begin{array}{c c}
0 & -1\\
0 & 1\\
1 & 0
\end{array}\right).
}

Given $\gamma \in (0,1)$, $\tau \geq 0$ and $\varepsilon \in (0, 1)$, define
\als{
r^{a_1} = \left( \begin{array}{c}
-\frac{\tau}{4}(2\varepsilon + \gamma - 1)\\
\frac{\tau(1 - \gamma^2)}{4\gamma}\\
-\frac{\tau(1 - \gamma^2)}{4\gamma}
\end{array}\right), & & r^{a_2} = \left( \begin{array}{c}
\frac{\tau}{4}(2\varepsilon + \gamma - 1)\\
\frac{\tau(1 - \gamma^2)}{4\gamma}\\
-\frac{\tau(1 - \gamma^2)}{4\gamma}
\end{array}\right),
}
which gives $V^* = \frac{\tau}{4}\pb{2(1 - \varepsilon), \frac{1 + \gamma}{\gamma}, -\frac{1 + \gamma}{\gamma}}\tra$ and $U^* = \frac{\tau}{4}\pb{ 2\varepsilon, \frac{1 - \gamma}{\gamma}, -\frac{1 - \gamma}{\gamma} }\tra$.
Given that
\als{
(\Pkern^{a_1} - \Qkern^{a_1}\proj)U^* = \left( \begin{array}{c}
U^*_2 + U^*_2\\
U^*_2 - U^*_3\\
U^*_3 - U^*_2
\end{array}\right), & & (\Pkern^{a_2} - \Qkern^{a_2}\proj)U^* = \left( \begin{array}{c}
U^*_3 + U^*_3\\
U^*_2 - U^*_3\\
U^*_3 - U^*_2
\end{array}\right),
}
and that $U^*_2 = -U^*_3$, we have $\normb{(\Pkern - \Qkern\proj)U^*} = 2U_2^* = \frac{1 - \gamma}{2\gamma}\tau$, which gives $\frac{2\gamma}{1 - \gamma}\normb{(\Pkern - \Qkern\proj)U^*} = \tau$.
It can be seen also that $\pihat(x_1) = a_2$ (and that the policy in $x_2$ and $x_3$ is irrelevant), so $\normb{V^* - V^{\pihat}} = (1 - \varepsilon)\tau$, since $r^{a_2}(x_1) + \gamma V^*_3 = -V^*_1 = -(1 - \varepsilon)\frac{\tau}{2}$.
We note in passing that $\normb{V^* - U^*} = \ab{\frac{\tau}{2} - \tau\varepsilon}$ and that $\normb{V^{\pihat} - U^*} = \frac{\tau}{2}$.
\end{proof}
\Cref{prop:infNormHarsh} is based on the natural argument that the model does not need to be good in states that are not visited by an optimal policy: i) we can extend the example in \cref{prop:tightness} with an initial state with two actions: ``stay'', or  ``go to the fork state''; ii) we pick the value of staying to be higher than the value of going to the fork state; iii) we pick an accurate model at the initial state, so that both $\pihat$ and $\pi^*$ choose to stay there (rather than go to the fork state).
The policy error is zero when we take $\mu$ that puts measure one on the initial state, however $\pihat$ is still near-sighted in the fork state, and it suffers the supremum norm error outlined in \cref{prop:tightness}.

\begin{restatable}[The supremum norm is harsh]{proposition}{propInfNormHarsh}
\label{prop:infNormHarsh}
There exist $\Pkern, \Qkern$ and $\proj$ \st{} for every $\gamma \in (0, 1)$ and $\tau > 0$, there exists $r \in \sn{V}^{\sn{A}}$ (the rewards scale with $\frac{1 - \gamma^2}{\gamma}\tau$) \st{} $\Lip(\Qkern\proj) < \infty$, $\Lip(\projA\Qkern) \leq 1$, $\normb{V - V^{\pihat}}_{\infty} = \tau$ and $\normb{V^* - V^{\pihat}}_{\mu,p} = \normb{V^* - V^{\pihat}}_{\xi,p} = 0$ where $\mu$  and $\xi$ are stationary \wrt{} $\pi^*$ and $\pihat$, respectively.
\end{restatable}
\begin{proof}
Pick any $\tau' > 0$.
Consider $\Pkern, \Qkern, r$ as in \cref{prop:tightness}, for the choice of $\varepsilon = \frac{1}{2}$ and $\tau' = \frac{\tau}{2}$.
Add a state, $x_4$, to $\sn{X}$, redefine $\proj V \eqdef (V(x_2), V(x_3), V(x_4))\tra$, let $\Pkern^{a_1}_{4,4} = 1$, $\Pkern^{a_2}_{4,1} = 1$, $\Qkern^{a}_{i,4} = 0$ for all $a$ and $i \neq 4$, and let also $\Qkern^{a}_{4,i} = \Pkern^{a}_{4,i}$ for all $a, i$.
Finally, let $r^{a_1}(x_4) = 2(1 - \gamma)\tau'$ and $r^{a_2}(x_4) = 0$

Thus, $V^*_4 = 2\tau'$, $\pi^*(x_4) = a_1$, $U^*_1 = V^*_1 = \tau'$, $\pihat(x_4) = a_1$ and $U^*_4 = 2\tau'$.
Moreover, the distribution $\mu$ ($\xi$) defined by $\mu(x_4) \eqdef 1$ ($\xi(x_4) \eqdef 1$) is stationary \wrt{} $\pi^*$ ($\pihat$).
This gives $\normb{V^* - V^{\pihat}}_{\mu,p} = \normb{V^* - V^{\pihat}}_{\xi,p} = 0$ as desired, and $\normb{V^* - V^{\pihat}}_{\infty} = \tau'$, which implies the result.
\end{proof}

To conclude, we present \cref{thm:muNormBoundInfty}.

\begin{restatable}[Weighted supremum norm bound for the policy error in $L^p(\mu)$ norm]{theorem}{thmMuNormBoundInfty}
\label{thm:muNormBoundInfty}
Let $\pihat$ be the policy derived from the factored linear model defined using \eqref{eq:pihatdef} and \eqref{eq:ustardef}.
If \cref{ass:boundedQkern,ass:Lyapunov} holds for the weighted supremum norm over $\sn{V}^{\sn{A}}$ and $\sn{W}^{\sn{A}}$, then
\als{
\normb{V^* - V^{\pihat}}_{\mu,p} \leq \norm{\nu}_{\mu,p} \pb{ \varepsilon(V^*) + \varepsilon(V^{\pihat}) },
}
where $\varepsilon(V) = \min(\varepsilon_1(V),\varepsilon_2)$, and
\begin{align*}
\varepsilon_1(V) &= \gamma \normb{(\Pkern - \Qkern\proj)V}_{\infty, \nu} + \frac{B\gamma^2}{1 - \beta_{\eta, \projA\Qkern}} \, \normb{\proj(\Pkern - \Qkern\proj)V}_{\infty, \eta} \\
\varepsilon_2 &= \frac{\gamma}{1 - \beta_{\nu, \Pkern}} \, \mnormb{(\Pkern  - \Qkern\proj)U^*}_{\infty, \nu}.
\end{align*}
\end{restatable}

\begin{proof}(of \cref{thm:muNormBoundInfty})
Since
\[
    \norm{V}_{\mu} \leq (\mu(\nu^p))^\frac{1}{p}\norm{\ab{V}^p}_{\infty,\nu^p} = \norm{\nu}_{\mu,p} \, \norm{V}_{\infty,\nu},
\]
we can apply \cref{thm:weightedInftyNormBound} to obtain the result.
\end{proof}

\section{Issues with bounding $\normb{U^* - V^*}_{\infty}$ instead of the policy error}
\label{sec:valueBoundIssues}

As we indicated in \cref{sec:baseline}, \citet{ormoneit2002kernel,barreto2011reinforcement,barreto2011computing,precup2012online,barreto2014policy,barreto2014practical}
\footnote{To be precise, the proof of \Citet{ormoneit2002kernel}'s Theorem 2 implies that this quantity converges to zero as the model error converges to zero (their analysis confounds the estimation and approximation errors).
Their Theorem 3, using an additional argument, is concerned with the probability of choosing a suboptimal action when using the approximate model.}
bound $\norm{V^* - \widetilde{V}}_{\infty}$ (not the policy error).
We can show by counterexample that this is not is not the correct quantity to bound in order to understand the quality of $\pihat$, and that the policy error should be bounded instead.
The recipe for constructing the counterexample proving this \cref{prop:errorGaps} is similar to the one used in \cref{prop:tightness}.

\begin{restatable}[Controlling only $\normb{U^* - V^*}_{\infty}$ is not enough]{proposition}{propErrorGaps}
\label{prop:errorGaps}
There exist $\Pkern, \Qkern$ and $\proj$ \st{}
satisfying  \cref{ass:boundedQkern,ass:boundedNorm} such that
for every $\gamma \in (0, 1)$, $\tau_1 \geq 0$ and $\tau_2 \geq 0$
 there exists a reward function $r \in \sn{V}^{\sn{A}}$ with $\norm{r}_{\infty}\le 2(\tau_1\vee \tau_2)/\gamma$ 
 \st{}
  $\normb{V^* - U^*} = \tau_1$, $\normb{V^{\pihat} - U^*} = \tau_2$ and
$\normb{V^* - V^{\pihat}}_{\infty} = \tau_1 + \tau_2$.
The rewards scale proportionally to $\frac{1 - \gamma}{\gamma}\max\cb{\tau_1, \tau_2}$.
\end{restatable}

\begin{proof}
The set of states is $\sn{X} = \cb{x_1,\ldots,x_3}$, the set of actions is $\sn{A} = \cb{a_1, a_2}$ and the transition probability kernel is specified by $\Pkern$ as follows:
\als{
\Pkern^{a_1} = \left( \begin{array}{c c c}
0 & 1 & 0\\
0 & 1 & 0\\
0 & 0 & 1
\end{array}\right), & & \Pkern^{a_2} = \left( \begin{array}{c c c}
0 & 0 & 1\\
0 & 1 & 0\\
0 & 0 & 1
\end{array}\right)
}
We let $\sn{W} = \rl{}$ and $\proj V \eqdef (V(x_2))\tra$.

Given $\gamma \in (0,1)$, $\tau_1 \geq 0$ and $\tau_2 \geq 0$, define $\tau_{\max} = \max\cb{\tau_1, \tau_2}$ and
\als{
r^{a_1} = \left( \begin{array}{c}
\tau_1 \\
\frac{1 - \gamma}{\gamma}\pb{\tau_1 + \tau_{\max}}\\
-\frac{1 - \gamma}{\gamma}\tau_2
\end{array}\right), & & r^{a_2} = \left( \begin{array}{c}
\tau_1 + \tau_{\max} \\
\frac{1 - \gamma}{\gamma}\pb{\tau_1 +\tau_{\max}}\\
-\frac{1 - \gamma}{\gamma}\tau_2
\end{array}\right),
}
which gives $V^* = \pb{2\tau_1 + \tau_{\max}, \frac{\tau_1 + \tau_{\max}}{\gamma}, -\frac{\tau_2}{\gamma} }\tra$.

Next, we construct $\Qkern$.
First, we set $\Qkern^{a_1}_{1,1} = \Qkern^{a_2}_{1,1} = 0$.
We want $u^* = (V^*_2)$ (which means $\proj V \eqdef V(x_2)$ for $V \in \sn{V}$), so we set $\Qkern^{a_1}_{2,2} = \Qkern^{a_2}_{2,2} = 1$.
Since $U^* = MT_{\Qkern}u^*$, we have $U^*_2 = u^*_2$ and $U^*_1 = \max_a r^a(x_1) = \tau_1 + \tau_{\max} = V^*_1 - \tau_1$.
We choose $\Qkern^a_{3,1} = -\frac{\tau_2}{\tau_1 + \tau_{\max} + \ind{\tau_{\max} = 0}}$ (that is, if $\tau_{\max} = 0$, we set $\Qkern^a_{3,1} = 0$), so that
\[
U^*_3 = \max_a r^a(x_3) + \gamma \Qkern^a_{3,2} u^*_2 = V^*_3.
\]
To summarize,
\als{
\Qkern^{a_1} = \Qkern^{a_2} = \left( \begin{array}{c c}
0\\
1\\
-\frac{\tau_2}{\tau_1 + \tau_{\max} + \ind{\tau_{\max} = 0}}
\end{array}\right).
}
At this point, we can see that $\Qkern$ does not depend on $\gamma$, that $\Lip(\Qkern\proj) < \infty$ and that $\Lip(\projA\Qkern) = 1$.

The policies obtained are given by $\pi^*(x_1) = 1$ and $\pihat(x_1) = 2$, while the choices for other states are irrelevant.
This gives $V^{\pi} = \pb{ r^{a_2}(x_1) + \gamma V^*_3, V^*_2, V^*_3 }\tra$, so that
\[
\norm{V^* - V^{\pihat}} = V^*_1 - r^{a_2}(x_1) - \gamma V^*_3 = \tau_1 + \tau_2.
\]
Moreover,
\[
\norm{V^* - U^*} = \ab{V^*_1 - (V^*_1 - \tau_1)} = \tau_1,
\]
and
\[
\norm{V^{\pihat} - U^*} = V^*_1 - \tau_1 - (r^{a_2}(x_1) + \gamma V^*_3) = \tau_2,
\]
which concludes the proof.
\end{proof}

\paragraph{Comparison to ADP.}
When a simulator of the true MDP is available (a case studied in the so-called simulation optimization literature), one can imagine to be able to compute a policy that is greedy in the true MDP with respect to some fixed value function up to an arbitrary accuracy at any given state.
 \citet[][Theorem 1]{singh1994upper}, \citet[][Theorem 4.1]{de2003linear}, \citet[][Proposition 3.1]{bertsekas2012weighted}
 and \citet[][]{Grunewalder2011:arxiv}
 bound the suboptimality of the resulting policy.

  A potentially more useful result is to bound the suboptimality of a policy derived from an action value-function (derived
 from a model).
Although we were unable to locate such a result in the literature, it can be derived using the techniques in the above-mentioned works.
These two results are summarized as follows:
\begin{theorem}[ADP policy error bounds]
\label{thm:adp}
For any $\widetilde{V} \in \sn{V}^{\sn{A}}$, if $\tilde{\pi} \eqdef G\widetilde{V}$, then
\begin{equation}
    \normb{V^* - V^{\tilde{\pi}}}_{\infty} \leq \frac{2(1 + \gamma)}{1 - \gamma}\normb{\widetilde{V} - T_{\Pkern}V^*}_{\infty}. \label[equation]{eq:generalADP}
\end{equation}
Alternatively, for any $\widetilde{V} \in \sn{V}$, if $\tilde{\pi} \eqdef GT_{\Pkern}\widetilde{V}$, then
\begin{equation}
    \normb{V^* - V^{\tilde{\pi}}}_{\infty} \leq \frac{2\gamma}{1 - \gamma}\normb{\widetilde{V} - V^*}_{\infty}. \label[equation]{eq:specificADP}
\end{equation}
\end{theorem}
To finish the discussion of the relevance of bounding the deviation $\normb{U^*-V^*}_{\infty}$ in a model-based setting,
    from \eqref{eq:specificADP} (by choosing $\widetilde{V}'=U^*$) we see
    that controlling this deviation would suffice if the policy was derived using the true model.
When this is not an option, one needs to fall back to \eqref{eq:generalADP}, calling for bounding
the difference between the action-value fixed point of a model and the action-value fixed point
of the true model.
To that end, we could use \cref{thm:inftyNormBound}, but the resulting bound would scale with $\frac{\gamma}{(1 - \gamma)^2}$, while both \cref{thm:baseline} and our later results scale with $\frac{\gamma}{1-\gamma}$ only.
Therefore, it is better to use \cref{thm:inftyNormBound} directly to bound the policy error.

\section{Additional remarks about $\proj$}

In this section, we carry out a brief discussion about the case when $\proj$ is a point evaluator (in which case $\Lip(\proj) = \Lip(\projA)$).

In supremum norm, we were able to use that $\Lip(\proj) \leq 1$ to get from \cref{prop:projAQkernContraction} that if $M' T_{\projA \Qkern}$ is a contraction and $\Lip(\Qkern) < \infty$, then some power of $MT_{\Qkern\proj}$ is a contraction.
In weighted supremum norm, $\Lip(\proj) = \max_i \frac{\eta_i}{\nu(x_i)}$, and, if this quantity is finite, some power of $MT_{\Qkern\proj}$ is a contraction as well.
\begin{proposition}
\label{prop:projAQkernContraction}
If $\Lip(M)\Lip(\Qkern)\Lip(\proj) < \infty$, $\Lip(M') \leq 1$, and $\Lip(\projA\Qkern) \leq 1$, $(MT_{\Qkern\proj})^{m}$ is a contraction for all $m$ large enough.
\end{proposition}
\begin{proof}
We have that $\Lip(MT_{\Qkern}) = \Lip(M) \gamma \Lip(\Qkern) \eqdef B' < \infty$, and $\Lip(M' T_{\projA\Qkern}) \leq \Lip(M') \gamma \leq \gamma$.
For $m \geq 0$, $(MT_{\Qkern\proj})^{m+1} = MT_{\Qkern}(M' T_{\projA\Qkern})^m\proj$, so $\Lip( (MT_{\Qkern\proj})^{m+1} ) \allowbreak \leq B'\gamma^{m+1}\Lip(\proj)$.
Given $m$ \st{} $B'\gamma^m < 1$,  thus $(MT_{\Qkern\proj})^{m'}$ is a contraction for all $m' \geq m$.
\end{proof}

In the case of $L^p(\mu)$ norms, \cref{prop:LipProj} gives us the form for $\Lip(\proj)$.
Having noted that $\sn{I}$ indexes a measurable subset of $\sn{X}$ (since $\proj$ is a point evaluator), we extend $\rho$ to $\sn{X}$ by $\rho(X) \eqdef \rho({i \in \sn{I} : x_i \in X})$ for measurable $X \subseteq \sn{X}$.
We denote absolute continuity of (the extension of) $\rho$ \wrt{} $\mu$ by $\rho \ll \mu$.
\begin{restatable}{proposition}{propLipProj}
\label{prop:LipProj}
Assume that $\proj$ is a point evaluator, and that the norm overs $\sn{V}$ and $\sn{W}$ are respectively an $L^p(\mu)$ and an $L^p(\rho)$ norm.
If $\rho \ll \mu$, then $\Lip(\proj) = \normb{\frac{d\rho}{d\mu}}^{\frac{1}{p}}_{\infty}$, otherwise $\Lip(\proj) = \infty$.
\end{restatable}
\begin{proof}
Thanks to the linearity of $\proj$, we have
\[
\Lip(\proj) = \sup_{V \neq 0} \frac{\norm{\proj V}_{\rho,p}}{\norm{V}}.
\]
From absolute continuity we get that $\int\ab{V(x)}^p d\rho(x) = \int \ab{V(x)}^p \pb{\frac{d\rho(x)}{d\mu(x)}}d\mu(x)$, and from H\"older's inequality we get
\[
    \int \ab{V(x)}^p \pb{\frac{d\rho(x)}{d\mu(x)}}d\mu(x) \leq \normb{\frac{d\rho}{d\mu}}_{\infty} \cdot \int \ab{V(x)}^p d\mu(x),
\]
which implies that $\Lip(\proj) \leq \normb{\frac{d\rho}{d\mu}}^{\frac{1}{p}}_{\infty}$.

To show that the upper-bound above is tight, we can see that
\als{
\Lip(\proj )
&= \sup_{V \neq 0} \frac{\norm{\proj V}_{\rho,p}}{\norm{V}}\\
&\geq \sup_{\substack{X \subseteq \sn{X}\\ \mu(X) > 0}} \frac{\int_X d\rho(x)}{\int_X d\mu(x)} = \sup_{\substack{X \subseteq \sn{X}\\ \mu(X) > 0}} \frac{\rho(X)}{\mu(X)},
}
because we can restrict $V$ to the indicator function of an $X \subseteq \sn{X}$.
If $\rho$ is not absolutely continuous \wrt{} $\mu$, then there exists $X$ \st{} $\mu(X) = 0$ and $\rho(X) > 0$, which implies that $\Lip(\proj ) = \infty$.
Otherwise, $\mu(X) = 0 \Rightarrow \rho(X) = 0$ for all $X \subseteq \sn{X}$, and
\[
\sup_{\substack{X \subseteq \sn{X}\\ \mu(X) > 0}} \frac{\rho(X)}{\mu(X)}
= \sup_{X \subseteq \sn{X}} \frac{\rho(X)}{\mu(X)}
= \normb{\frac{d\rho}{d\mu}}_{\infty},
\]
which concludes the proof.
\end{proof}

Interestingly, an unbounded $\Lip(\proj)$ can lead to $\Lip(MT_{\Qkern\proj}) = \infty$, as stated by \cref{prop:TQRContraction}, and, yet, if \cref{ass:boundedNorm} $MT_{\Qkern\proj}$ still has a fixed point, and, provided that \cref{ass:boundedQkern} is met in addition, we can still obtain performance bounds for the policy error of $\pihat$.
\begin{restatable}{proposition}{propTQRContraction}
\label{prop:TQRContraction}
Assume $\proj$ is a point evaluator, that the norms over $\sn{V}$ and $\sn{W}$ are, respectively, an $L^p(\mu)$ and an $L^p(\rho)$ norm, and that the norms over $\sn{V}^{\sn{A}}$ and $\sn{W}^{\sn{A}}$ are the corresponding mixed norms defined using $M_{|\cdot|}$.
If $\mu(\cb{x_i : i \in \sn{I}}) = 0$, then for all $m \geq 0$, the following holds: If $\Lip(MT_{\Qkern}(M' T_{\projA \Qkern})^{m}) > 0$ then $\Lip((MT_{\Qkern\proj})^{m+1}) = \infty$.
\end{restatable}

\begin{proof}
Let $S \eqdef \cb{x_i : i \in \sn{I}}$.
When $\mu(S) = 0$, we have $\rho \not\ll \mu$ and $\Lip(\proj) = \infty$.
We will show that for any $m \geq 0$ either $\Lip((MT_{\Qkern\proj})^{m+1}) = 0$ or $\Lip((MT_{\Qkern\proj})^{m+1}) = \infty$.
To that end, define $(Zu)(x_i) \eqdef u_i$ for $i \in \sn{I}$ (for simplicity, assume that, for all $x_i, x_j \in S$, $x_i = x_j \Rightarrow i = j$), and let $(Zu)(x) \eqdef 0$ for $x \notin S$.
Then $\sup_{u \in \sn{W}}\normb{Zu} = 0$ and $\proj Zu = u$ for all $u \in \sn{W}$.
The definitions then give:
\als{
\Lip((MT_{\Qkern\proj})^{m+1})
&= \Lip((MT_{\Qkern\proj})^{m}T_{\Qkern}\proj)\\
&= \sup_{\substack{V, V' \in \sn{V} : \\ V\neq V' }}\frac{\normb{(MT_{\Qkern\proj})^{m}MT_{\Qkern}\proj V - (MT_{\Qkern\proj})^{m}MT_{\Qkern}\proj V'}}{\normb{V - V'}} \\
&\geq \sup_{\substack{u, u' \in \sn{W} : \\ u\neq u' }}\frac{\normb{(MT_{\Qkern\proj})^{m}MT_{\Qkern}u - (MT_{\Qkern\proj})^{m}MT_{\Qkern}u'}}{\normb{Zu - Zu'}},
}
which is unbounded unless $\Lip((MT_{\Qkern\proj})^{m}MT_{\Qkern}) = 0$.
To conclude, we observe that $(MT_{\Qkern\proj})^{m}MT_{\Qkern} = MT_{\Qkern}(M' T_{\projA\Qkern})^m$, since $\proj$ is a point evaluator.
\end{proof}

\end{document}